\documentclass[journal]{IEEEtran}
\usepackage{cite}
\usepackage{picinpar}
\usepackage{balance}
\usepackage[hidelinks]{hyperref}
\usepackage[pdftex]{graphicx}
\graphicspath{{figs/pdf/}{figs/jpeg/}}
\DeclareGraphicsExtensions{.pdf,.jpeg,.png}
\usepackage{amssymb,amsmath}
\usepackage{amsthm}
\usepackage{amsfonts}
\usepackage{subeqnarray}
\usepackage{algorithm}
\usepackage{algorithmicx}
\usepackage{algpseudocode}
\usepackage{subfigure}
\usepackage{graphicx}
\usepackage{epstopdf}
\usepackage{amsfonts,amssymb}
\usepackage{bm}
\usepackage{mathtools}
\usepackage{makecell}
\newtheorem{theorem}{Theorem}

\newtheorem{lemma}{Lemma}

\usepackage{threeparttable}
\usepackage{mathrsfs}
\usepackage{subfigure}
\usepackage{verbatim}
\usepackage{multirow}
\usepackage{tabu}
\usepackage{booktabs}
\usepackage{color,xcolor}
\usepackage{enumitem}
\usepackage{array}
\usepackage{tabularx}
\usepackage{supertabular}

\begin{document}

\title{FSMP: A Frontier-Sampling-Mixed Planner for Fast Autonomous Exploration of Complex and Large \\ 
3-D Environments}

\author{Shiyong~Zhang, 
        Xuebo~Zhang$^{\star}$,
        Qianli~Dong,
        Ziyu~Wang, 
        Haobo~Xi, and~
        Jing~Yuan        
       
\thanks{This work was supported in part by Natural Science Foundation of China under Grant Number 62303249 and 62293513/62293510, in part by the China Postdoctoral Science Foundation-Tianjin Joint Support Program under Grant Number 2023T013TJ. (\emph{Corresponding author: Xuebo Zhang.})}
\thanks{The authors are with the Institute of Robotics and Automatic Information System, College of Artificial Intelligence, Nankai University, Tianjin, China. (e-mail: syzhang@mail.nankai.edu.cn).}
}

\markboth{}
{}

\maketitle

\begin{abstract}
In this paper, we propose a systematic framework for fast exploration of complex and large 3-D environments using micro aerial vehicles (MAVs). The key insight is the organic integration of the frontier-based and sampling-based strategies that can achieve rapid global exploration of the environment. Specifically, a field-of-view-based (FOV) frontier detector with the guarantee of completeness and soundness is devised for identifying 3-D map frontiers. Different from random sampling-based methods, the deterministic sampling technique is employed to build and maintain an incremental road map based on the recorded sensor FOVs and newly detected frontiers. With the resulting road map, we propose a two-stage path planner. First, it quickly computes the global optimal exploration path on the road map using the lazy evaluation strategy. Then, the best exploration path is smoothed for further improving the exploration efficiency. We validate the proposed method both in simulation and real-world experiments. The comparative results demonstrate the promising performance of our planner in terms of exploration efficiency, computational time, and explored volume.
\end{abstract}

\IEEEpeerreviewmaketitle

\section{Introduction}
\IEEEPARstart{A}{utonomous} exploration plays a pivotal role in many robotic applications, such as infrastructure inspection \cite{RAL_Xu_2023}, environmental modeling \cite{TIM_Li_2023,TIM_LIU_2024}, robot collaboration\cite{TRO_Zhou_2023}, search and rescue \cite{TIM_Zhang_2022}, among others. Benefiting from the high maneuverability, MAVs have been widely employed for exploring 3-D unknown environments. Nevertheless, the main challenge of autonomous exploration is that the robot is required to make decisions online for the navigation in unknown environments, which will become trickier when the employed MAV is resource-limited.

There are two mainstream strategies for addressing the autonomous exploration problem: 1) sampling-based exploration \cite{Bircher_2018, Duberg_2022, Schmid_2020, Xuetao_2021} and 2) frontier-based exploration \cite{yamauchi_1997, TIE_2024, RAL_2023, Zhou_2021, TII_2023}. 
For the sampling-based strategy, it explores the environment by randomly sampling viewpoints in the free space, which has great potential for exploring an independent region \cite{Bircher_2018}. However, this type of method usually gets stuck in local minima when the environment is complex and consists of many separate regions \cite{Schmid_2020}. Differently, frontier-based strategy starts by detecting frontiers of the whole environment and directs the robot to these regions to complete the exploration. Even this strategy can achieve the global evaluation of the exploration candidates, it usually results in robots going back and forth in the environment. To relieve the problem, the works proposed in \cite{TIE_2024, RAL_2023, Zhou_2021} compute the visit sequence of frontiers by solving the traveling salesman problem (TSP), which is known as an NP-hard problem. Therefore, the effort required for solving TSP will increase dramatically with the size of the environment, resulting in a formidable computational overhead for real-time path planning (especially in complex and large 3-D environments). It is intuitive to combine these two strategies in one unified framework \cite{Selin_2019, Dai_2020, TIM_chaoqun2020, Zhu_2021}. However, the critical problem is how to achieve an effective integration of different strategies. Existing methods generally adopt the sampling-based strategy for local exploration and utilize frontiers for global exploration planning when the local exploration finish. Practically, such a naive combination is still prone to local minima, leaving room for further improvement.

In this paper, we propose FSMP, a {\bf{F}}rontier-{\bf{S}}ampling-{\bf{M}}ixed exploration {\bf{P}}lanner, which can achieve real-time performance on-board an MAV. Specifically, we propose a fast field-of-view-based frontier detector called F$^3$D that only examines the minimal space every time. The detector can identify frontiers efficiently and accurately when new sensor measurements are received. To achieve global coverage, a road map is incrementally built by using the uniformly deterministic sampling method. With the road map in hand, the path from the current position of the robot to each exploration candidate and its corresponding motion cost can be queried directly, thereby facilitating the evaluation of all exploration candidates in a global context. To further improve the efficiency, a two-stage path planner is devised. First, we propose a lazy evaluation strategy that benefits from the breadth-first search (BFS) on the road map. This strategy allows us to obtain the global optimal exploration path, while avoiding the need to compute the utility of all candidates. Second, we optimize the best exploration path for improving its smoothness and time allocation, which can further reduce the exploration time. 

We evaluate FSMP by comparing it with existing popular methods in simulations. Furthermore, the proposed planner is successfully integrated into a fully autonomous MAV and its performance is validated in real-world environments. The comparative results demonstrate that our method outperforms the existing ones. The main advantages and contributions of this work are summarized as follows:
\begin{enumerate}
    \item We propose a frontier-sampling-mixed planner for fast autonomous exploration of complex and large 3-D environments, which can overcome the shortcomings of frontier-based and sampling-based methods, and achieve real-time performance on-board an MAV.
	\item To speed up the frontier detection, a fast field-of-view-based frontier detector (F$^3$D) with the guarantee of completeness and soundness is elaborately designed. F$^3$D can identify frontiers quickly and thus facilitate the following exploration process.
	\item To achieve complete exploration of the environment, we propose to incrementally build and maintain a road map using uniformly deterministic sampling. The road map can return a feasible path towards each exploration candidate directly and efficiently.
    \item We propose a two-stage path planner to rapidly compute the global optimal exploration path on the road map using the devised lazy evaluation strategy, and then smooth the path for further improving the exploration efficiency.
\end{enumerate}

The remainder of the paper is organized as follows. In Section II, we review related work on autonomous exploration in three aspects. We define notions, formulate the exploration problem and present an overview of our framework in Section III. In Section IV, we introduce the proposed exploration planner in detail. We evaluate the performance of our method comprehensively in Section V. Finally, we conclude the paper and outline the future work in Section VI.

\section{Related Work}\label{section:Related-Work}

\subsection{Frontier-based Exploration Algorithms}\label{sub_section:Frontier-based-Exploration}
The frontier-based exploration is the most classic strategy for robotic exploration and it is first proposed by Yamauchi \cite{yamauchi_1997}. The frontier is a boundary which separates the known areas from those unknown in the environment. Commanding the robot to the frontier region can observe new areas and thus complete the exploration. Therefore, frontier detection is the key component of frontier-based exploration algorithms \cite{yamauchi_1997, TIE_2024, RAL_2023, Brunel_2021, Zhou_2021}. The naive frontier detection method \cite{yamauchi_1997} needs to process the whole map data every time it is invoked, which is inefficient, especially in 3-D scenarios. Keidar \emph{et al.} \cite{IJRR2014} proposed two promising algorithms for frontier detection, i.e., Wavefront Frontier Detector (WFD) and Fast Frontier Detector (FFD). Instead of processing the whole map data in every iteration, WFD examines the known part of the map, which can effectively reduce the computational overhead. However, WFD will degenerate into the naive one when the map is extended. Different from WFD, FFD only examines the contour that consists of the sensor readings. Even though FFD can be executed fast, we observe it has two disadvantages. First, FFD should be invoked after each scan which may waste calculations. Second, FFD is a sensor-based method and it is only suitable for 2-D lasers. Similar to FFD, \cite{Quin_2021} proposed the Frontier-Tracing Frontier Detection (FTFD) algorithm that only needs to examine the perimeter of the field-of-view (FOV) of the sensor after each scan. In \cite{Quin_2014}, Quin \emph{et al.} proposed the Expanding-Wavefront Frontier Detector (EWFD), which just needs to search newly observed regions based on the assumption that the entropy of each grid cell can only decrease. However, the assumption will not hold when there are moving obstacles in the environment or the sensor has a non-negligible noise. To speed up frontier detection, the axis-aligned bounding box (AABB) is used in \cite{Zhou_2021}. Specifically, the AABB of the observations of the sensor will be recorded. Once a free cell in the AABB is found the region growing (RG) algorithm will be invoked for detecting new frontier cells. Nevertheless, it still scans unnecessary areas that are lying within the AABB.

\subsection{Sampling-based Exploration Algorithms}\label{sub_section:Sampling-based-Exploration}
For sampling-based autonomous exploration \cite{Bircher_2018, Respall_2021, Schmid_2020, Xuetao_2021}, it starts by randomly sampling exploration candidates (also known as viewpoints) in the free space of the environment. Then, the viewpoints will be evaluated for the determination of the exploration target. Inspired by the next-best-view (NBV) approaches \cite{Connolly_1985}, Bircher \emph{et al.} \cite{Bircher_2018} proposed the Next-Best-View Planner (NBVP) which runs in a receding horizon fashion. NBVP utilizes the rapid-exploring random tree (RRT) to generate random viewpoints and only executes the first segment of the best branch in every iteration. It is proven that NBVP is efficient for exploring simple individual regions. Nevertheless, when the environment is complex and composed of multiple sub-regions, NBVP will struggle in a local narrow area or terminate prematurely. To avoid getting stuck in the local area, Schmid \emph{et al.} \cite{Schmid_2020} proposed a informative path planner (IPP) to build a large-size single tree during exploration which can capture the global information of the environment, and thus avoids getting stuck in a local area. Yet, the major limitation of this method is that it needs a lot of computational resources to maintain the large tree. Different from the above methods, the work proposed in \cite{Xuetao_2021} only maintains the history trajectory that the robot has executed in earlier iterations. When the robot has fully explored its vicinity area, it will trace back the history trajectory for finding a valid exploration target. Tracing back previous trajectories is effective in avoiding getting stuck in local minima, but it is not efficient for the robot to transit to different regions.

\subsection{Hybrid Frontier-based and Sampling-based Algorithms}\label{sub_section:Frontier-sampling-mixed-Exploration}
There are also algorithms that incorporate frontier-based and sampling-based methods into a unified framework \cite{Selin_2019, Dai_2020, TIM_chaoqun2020, Zhu_2021}. Selin \emph{et al.} \cite{Selin_2019} proposed the Autonomous Exploration Planner (AEP), which utilizes NBVP for local exploration and the frontier-based method for global exploration. When the robot gets stuck, the frontiers are used to direct the robot to global regions. Nevertheless, AEP only caches the location of the frontiers, and a separate path planner is required for computing the global exploration path. Similar to \cite{Selin_2019}, the work in \cite{Dai_2020} proposed a hybrid exploration framework that combining frontier-based and sampling-based methods. In their framework, the hybrid planner first extracts the frontiers of the environment map. Then, exploration candidates are sampled from the map frontiers and evaluated for determining the next best view. Wang \emph{et al.} \cite{TIM_chaoqun2020} proposed to incrementally build a road map and extract frontiers around the nodes of the road map. The information gain and the cost-to-go of each exploration candidate can be queried on the road map efficiently. However, the road map needs to be extended after each scan and there is no completeness guarantee for the frontier detection algorithm. In \cite{Respall_2021}, the authors proposed a history-ware exploration planner. The method maintains a graph structure that consists of the old viewpoints from earlier iterations. Once the robot gets stuck in the dead end of the environment, it will reseed the active nodes on the graph to find a collision-free path towards global frontiers.

In summary, both frontier-based and sampling-based exploration methods have their inherent disadvantages. It is intuitive to combine these two strategies in a unified framework for improving the efficiency of robotic exploration. Existing methods, however, either achieve this combination loosely or make decisions greedily in the local context. Therefore, we propose a novel frontier-sampling-mixed planner for fast exploration of complex and large 3-D environments with an MAV.

\section{Problem Formulation}\label{section:Problem-Statement-and-Framework-Overview}

\subsection{Definitions}\label{sub_section:Problem-Formulation1}

\begin{figure}[!t]
	\begin{minipage}[!htbp]{\linewidth}
		\centering
		\subfigure[6-connected]{
			\label{figure1_a} 
			\includegraphics[width=1.22in]{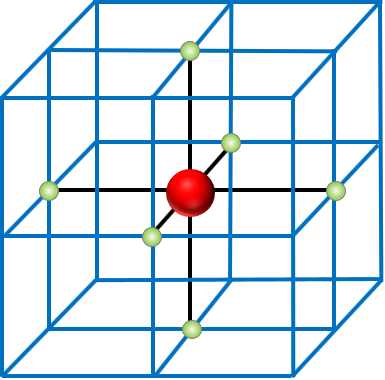}}
		\subfigure[26-connected]{
			\label{figure1_b} 
			\includegraphics[width=1.22in]{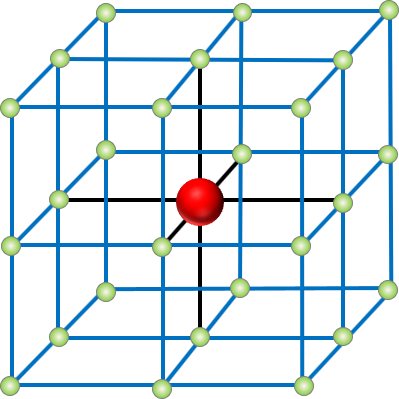}}
		\caption{Illustration of the adjacency relationship of the map voxels. The red sphere indicates a specific voxel, and the green spheres denote its neighbors.}
		\label{figure1} 
	\end{minipage}
\end{figure}

\textbf{Volumetric map}: a 3-D occupancy grid-based representation of the environment $\mathcal{V} \subset \mathbb{R}^3$, i.e., $\mathcal{M}$, where each grid cell in the volumetric map is named as a voxel $v$, $v$ $\in$ $\mathcal{V}$.

\textbf{Unknown space}: a sub-space of $\mathcal{V}$ that has not been covered yet by the onboard sensor of the robot, i.e., $\mathcal{V}_{un}$ ($\mathcal{V}_{un}$ $\subset$ $\mathcal{V}$).

\textbf{Free space}: a sub-space of $\mathcal{V}$ that has already been covered by the onboard sensor and it is not occupied by any obstacles, i.e., $\mathcal{V}_{free}$ ($\mathcal{V}_{free}$ $\subset$ $\mathcal{V}$).

\textbf{Occupied space}: a sub-space of $\mathcal{V}$ that has already been covered by the onboard sensor and it is occupied by obstacles, i.e., $\mathcal{V}_{occ}$ ($\mathcal{V}_{occ}$ $\subset$ $\mathcal{V}$).

\textbf{Frontier voxel}: a free voxel $v_f$ $\in$ $\mathcal{V}_{free}$ of the Volumetric map that has at least one unobserved voxel $v_{un}$ $\in$ $\mathcal{V}_{un}$ as its neighbor. As shown in Fig. \ref{figure1_a}, we designate the connectivity of voxels as 6-connected in this case.

\textbf{Frontier}: a set of connected frontier voxels, i.e., ${\rm{f}}$, ${\rm{f}} = \{v_f \in \mathcal{V}_{free} : \exists\; neighbor(v_f) \in \mathcal{V}_{un}\}$. In this case, the connectivity of frontier voxels is 26 (see Fig. \ref{figure1_b}).

\textbf{Region of interest (ROI)}: a sub-space of $\mathcal{V}$ that the frontier detector needs to search for detecting frontiers.

\textbf{Road map}: an undirected graph ${\cal R}$ that consists of nodes $N$ and edges $E$, i.e., ${\cal R} = \{N, E\}$. Essentially, the graph ${\cal R}$ is a sparse skeleton representation of $\mathcal{V}_{free}$.

\textbf{Exploration candidate}: a node of the road map that implies the location of informative regions needed to be explored.

\subsection{Problem Statement}\label{sub_section:Problem-Formulation2}
In this paper, we investigate the problem of autonomous exploration of 3-D environments with an MAV. The environment will be an unknown but spatially bounded space $\mathcal{V}  \subset \mathbb{R}^3$. We assume the MAV is equipped with a forward-looking depth sensor such as an RGB-D camera, which usually has a limited FOV. Let the state of the MAV be ${\rm{x}} := \{p, \xi\} \in \mathcal{X} \subset \mathbb{R}^4$, where $\mathcal{X}$ is the configuration space, $p = \{x,y,z\} \in \mathbb{R}^3$ is the position and $\xi \in \mathbb{R}^1$ is the yaw angle. At state ${\rm{x}}$, the MAV will collect its observation $O({\rm{x}})$ and integrate it into a probabilistic volumetric map $\mathcal{M}$ \cite{OctoMap_2013}. This map representation divides $\mathcal{V}$ into three sub-space: $\mathcal{V}_{un}$, $\mathcal{V}_{free}$, and $\mathcal{V}_{occ}$.

Initially, we have $\mathcal{V}_{un} = \mathcal{V}$. With the progress of autonomous exploration, it becomes $\mathcal{V}_{un} = \mathcal{V}\; \backslash\; \{\mathcal{V}_{free} \cup \mathcal{V}_{occ}\}$. Our ultimate objective is to explore the unknown space in the environment as much as possible, i.e., let $\mathcal{V}_{un} \to \emptyset$. To achieve the objective more efficiently, we aim to maximize the utility function ${\cal U}(\cdot)$ when computing the exploration path in each decision-making epoch, as follows:
\begin{eqnarray}\label{Eq_1}
{\nonumber\max \;{\kern 1pt} {\cal U}({\cal L}(\gamma)},\;{\cal I}({\gamma}))\\
{s.t.\left\{ {\begin{array}{*{20}{c}}
{{\gamma} \subset \mathcal{V}_{free}}\\
{g(\gamma) \ge 0}
\end{array}} \right.}
\end{eqnarray}
where $\gamma :[0,1] \to \mathcal{X}$ denotes the exploration path and it is defined as a sequence of states of the MAV, $\mathcal{L}(\gamma)$ is the motion cost for the robot following the exploration path, $\mathcal{I}({\gamma})$ is the expected exploration gain that the MAV can obtain along the path and ${g(\gamma) \ge 0}$ denotes the planned path should meet designated constraints. We can learn that ${\cal U}(\cdot)$ takes as input parameters the motion cost $\mathcal{L}(\gamma)$ and the expected exploration gain $\mathcal{I}({\gamma})$ and returns an evaluation result. The path that maximizes ${\cal U}$ will be executed to explore the environment.

\section{Methodology}\label{section:Methodology}

\subsection{Framework Overview}\label{sub_section:Framework-Overview}
\begin{figure}[!t]
	\centering
	\includegraphics[width=0.95\hsize]{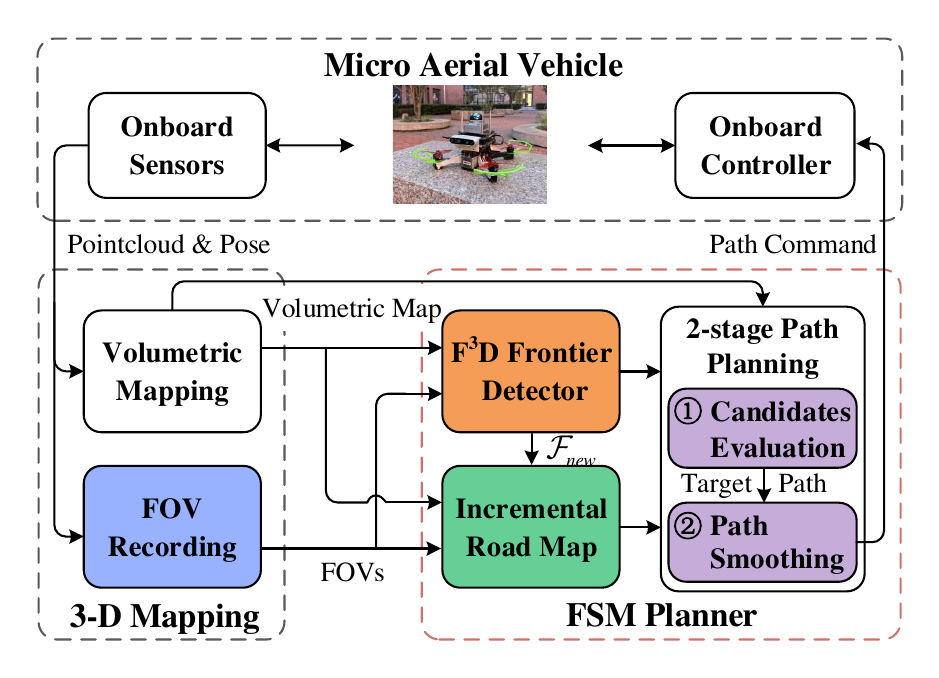}
	\caption{Overall system architecture of the proposed exploration framework.}
    \label{overall_framework}
\end{figure}
The overall system architecture of our exploration framework is shown in Fig. \ref{overall_framework}. The 3-D mapping module integrates the depth measurement and pose estimation into the volumetric map $\mathcal{M}$ while recording the FOVs of the onboard sensor. Whenever FSMP is invoked, the frontier detector first extracts new frontiers and eliminates the ones covered by the sensor (Section \ref{sub_section:Frontier-Detection}). After that, the incremental road map is extended, which can return a feasible path from the current position of the robot to any exploration candidate with its corresponding motion cost (Section \ref{sub_section:Global-Roadmap-Construction}). Exploration candidates are evaluated, and the one with the maximum exploration utility is selected as the next exploration target. The path towards the best exploration target will be further optimized for improving the efficiency (Section \ref{sub_section:Local-Path-Planning}). Finally, the best path is subscribed by the controller to command the MAV for accomplishing the exploration mission. The task will be terminated when there is no frontier exists.

\subsection{Fast Field-of-view-based Frontier Detection}\label{sub_section:Frontier-Detection}
\begin{figure}[!t]
	\centering
	\includegraphics[width=1.0\hsize]{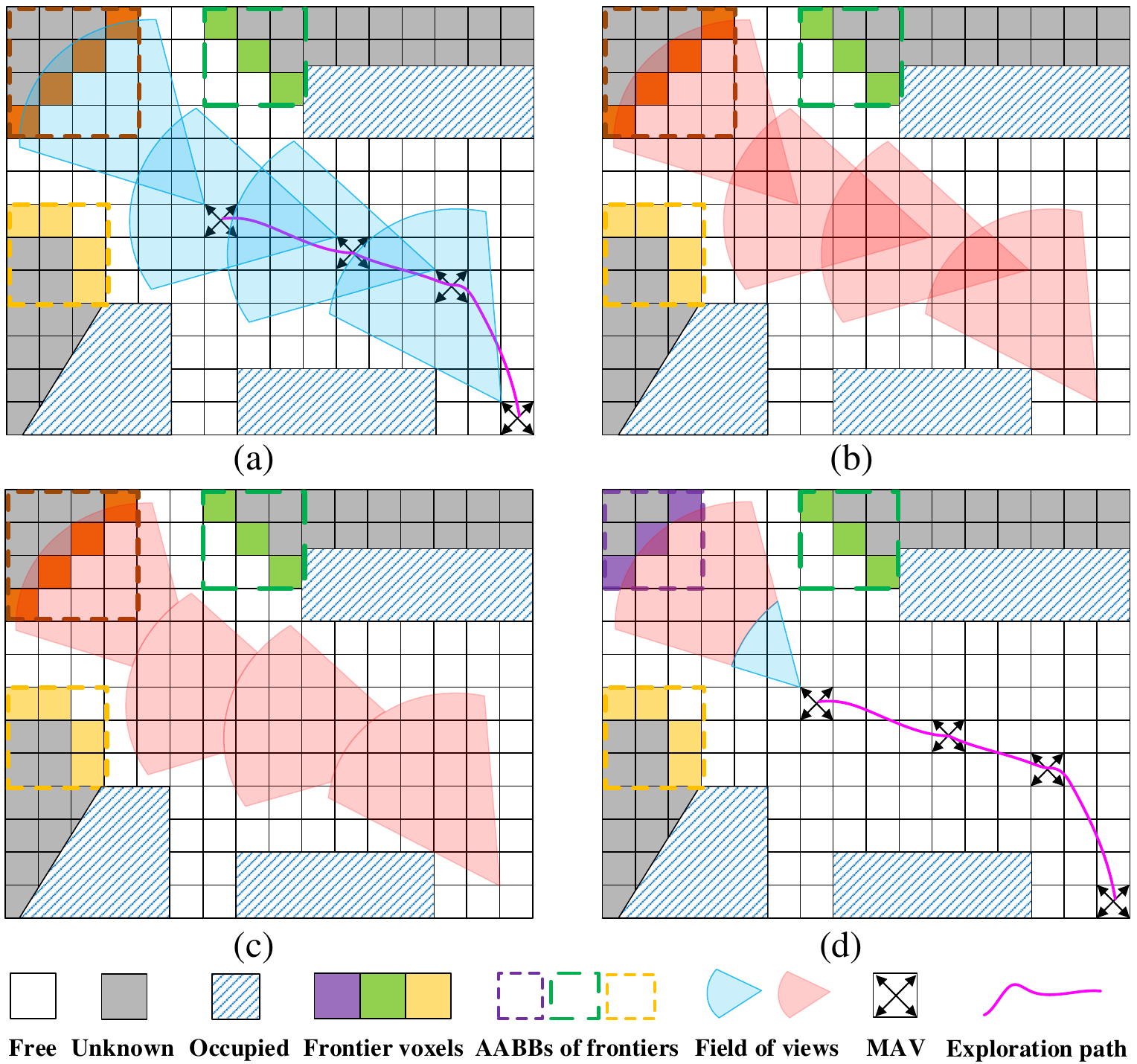}\vspace{-0.5cm}
	\caption{The 2-D illustration of the detection process of F$^3$D. (a) The MAV is performing the exploration path, meanwhile, the FOVs of its onboard sensor are recorded. (b) There are overlapping areas between the FOVs. (c) F$^3$D can avoid repeatedly examining the voxels of the overlapping areas. (d) In this case, new frontiers are detected in the last recorded FOV.} 
    \label{frontier_detection}
\end{figure}

\renewcommand{\algorithmicrequire}{\textbf{Input:}}
\renewcommand{\algorithmicensure}{\textbf{Output:}}
\begin{algorithm}[!t]
	\caption{Fast Field-of-view-based Frontier Detection}
	\label{Algorithm_1}
	\begin{algorithmic}[1]
		\Require {\emph{frontierSet}, ${\cal S}_t$, ${\cal F}_{t-1}$}
		\Ensure	{${\cal F}_t$}
		\State \emph{deleteSet} $\gets$ $\emptyset$;
		\State \emph{closedSet} $\gets$ $\emptyset$;
		\For{each {\rm{f}} $\in$ ${\cal F}_{t-1}$}
		\If {{\rm{f}} is intersecting with ${\cal S}_t$}
		\If {any \emph{voxel} $\in$ {\rm{f}} is changed}
		\State   \emph{deleteSet} $\gets$ \emph{deleteSet} $\cup$ {\rm{f}};
		\State   ${\cal F}_{t-1}$ $\gets$ ${\cal F}_{t-1}$ {\textbackslash}  {\rm{f}};
		\State   \emph{frontierSet} $\gets$ \emph{frontierSet} {\textbackslash}  {\rm{f}};
		\EndIf
		\EndIf
		\EndFor
		\State // Search for new frontiers in the ROI
		\For{each \emph{s} $\in$ ${\cal S}_t$}
		\For{each \emph{voxel} $\in$ \emph{s}}
		\If {\emph{voxel} $\in$ \emph{closedSet}}
		\State continue;
		\EndIf
		\If {any \emph{nbr} $\in$ \emph{Neighbor}(\emph{voxel}) belongs to a frontier {\rm{f}$_{nbr}$} $ \wedge $ any \emph{voxel} $\in$ {\rm{f}$_{nbr}$} is changed} 		
		\State   \emph{deleteSet} $\gets$ \emph{deleteSet} $\cup$ {\rm{f}$_{nbr}$};
		\State   ${\cal F}_{t-1}$ $\gets$ ${\cal F}_{t-1}$ {\textbackslash}  {\rm{f}$_{nbr}$};
		\State   \emph{frontierSet} $\gets$ \emph{frontierSet} {\textbackslash}  {\rm{f}$_{nbr}$};
		\EndIf
		\If {\emph{voxel} is a new frontier voxel}
		\State  \emph{newFrontier} $\gets$ \emph{Extract\_Frontier}(\emph{voxel}) (Alg. \ref{Algorithm_2});
		\State  ${\cal F}_{new}$ $\gets$ ${\cal F}_{new}$ $\cup$ \emph{newFrontier};
		\State  \emph{frontierSet} $\gets$ \emph{frontierSet} $\cup$ \emph{newFrontier};
		\State continue;
		\EndIf
		\State mark \emph{voxel} as \emph{closedSet};
		\EndFor
		\EndFor
		\State // Recheck the frontier voxels in the \emph{deleteSet}
		\For{each \emph{voxel} $\in$ \emph{deleteSet}}
		\If {\emph{voxel} is a frontier voxel $ \wedge $ \emph{voxel} $\notin$ \emph{frontierSet}}
		\State   \emph{newFrontier} $\gets$ \emph{Extract\_Frontier}(\emph{voxel}) (Alg. \ref{Algorithm_2});
		\State   ${\cal F}_{new}$ $\gets$ ${\cal F}_{new}$ $\cup$ \emph{newFrontier};
		\State   \emph{frontierSet} $\gets$ \emph{frontierSet} $\cup$ \emph{newFrontier};
		\EndIf
		\EndFor
		\State ${\cal F}_t$ $\gets$ \emph{Process\_Frontier}(${\cal F}_{new}$, ${\cal F}_{t-1}$);
	\end{algorithmic}
\end{algorithm}

\renewcommand{\algorithmicrequire}{\textbf{Input:}}
\renewcommand{\algorithmicensure}{\textbf{Output:}}
\begin{algorithm}[!t]
  	\caption{BFS-based Frontier Extraction}
  	\label{Algorithm_2}
  	\begin{algorithmic}[1]
  		\Require {\emph{voxel} // the starting point of the BFS}
  		\Ensure	{\emph{newFrontier} // the set of frontier voxels that are belonging to the same frontier}
        \State \emph{tempQueue} $\gets$ $\emptyset$;
        \State \emph{newFrontier} $\gets$ $\emptyset$;
        \State   ENQUEUE(\emph{tempQueue}, \emph{voxel});
        \While {\emph{tempQueue} is not empty}
        \State  \emph{voxel} $\gets$ DEQUEUE(\emph{tempQueue});
        \If {\emph{voxel} $\in$ \emph{closedSet}  $ \vee $ \emph{voxel} $\in$ \emph{frontierSet}}
        \State continue;
        \EndIf
        \If {\emph{voxel} is a frontier voxel}
        \State \emph{newFrontier} $\gets$ \emph{newFrontier} $\cup$ \emph{voxel};
        \For{each \emph{nbr} $\in$ \emph{Neighbors}\emph{(voxel)}}
        \If {\emph{nbr} $\notin$ \emph{closedSet}  $ \wedge $ \emph{voxel} $\notin$ \emph{frontierSet} $ \wedge $ \emph{voxel} is a frontier voxel}
        \State   ENQUEUE(\emph{tempQueue}, \emph{nbr});
        \EndIf
        \EndFor
        \EndIf
        \EndWhile
        \State \Return \emph{newFrontier}
  	\end{algorithmic}
\end{algorithm}

In this section, we propose our {\bf{F}}ast {\bf{F}}ield-of-view-based {\bf{F}}rontier {\bf{D}}etector, F$^3$D, which is an incremental frontier detector that only needs to examine a very minimal sub-space of the whole environment. Since in the frontier searching process only the information of the FOV of the onboard sensor is required for determining the ROI, F$^3$D is general for both 2-D and 3-D scenarios. In addition, F$^3$D is also not confined by the specific sensor type (e.g., laser range finders, sonars, depth cameras, etc.) since the FOV of any type of sensor can be defined in advance. Therefore, if the sensor type is determined, the ROI required by F$^3$D  can be computed directly.

As illustrated in Fig. \ref{frontier_detection}, every time the exploration path is performed by the MAV, the FOVs of its onboard sensor are recorded. Furthermore, when one frontier is detected, the axis-aligned bounding box (AABB) of this frontier is also computed and recorded. Once F$^3$D is invoked, we first extract the FOVs of the onboard sensor and store them in ${\cal S}_t$. Due to the map updating event only occurs in the ROI that consists of the FOVs, only the observations in the ROI can affect the frontier detection. Therefore, the frontier detector can be constrained to the ROI to maximize efficiency. As mentioned in \cite{Quin_2021}, the robot can not enter unknown space and all the known regions of the map are connected naturally. When the sensor covers unknown areas, the FOV of the sensor must span both known and unknown space (see Fig. \ref{frontier_detection}(a)), i.e., the FOV of the sensor must be intersecting with frontiers. Consequently, the old frontiers who are intersecting with the FOV of the onboard sensor should be removed from the frontier database (i.e., $\mathcal{F}_{t-1}$) and reexamined when searching for new frontiers (Alg. \ref{Algorithm_1}, lines 3-11 and lines 33-39). If they are still frontiers, they will be stored in $\mathcal{F}_t$. Otherwise, they will be marked as non-frontier voxels. Furthermore, when there are dynamic objects or the sensor data has a non-negligible noise, new frontiers will also appear in the location where no old frontiers exist. To guarantee the completeness of the frontier detection, F$^3$D will examine all voxels that are lying within the space covered by the sensor of the robot since last call.

It can be seen in Fig. \ref{frontier_detection}(b) that there are overlapping areas between the FOVs of the sensor, and this can cause inefficiencies in frontier detection especially when the map is updating frequently. In order to avoid rescanning the overlapping areas, F$^3$D maintains two sets in the frontier detection process. This strategy is inspired by the A$^*$ algorithm \cite{A_star_1968}. Specifically:
\begin{enumerate}
    \item frontierSet: a set of voxels that have already been marked as valid frontier voxels (Alg. \ref{Algorithm_1}, lines 8, 21, 37).
	\item closedSet: a set of voxels lying within the ROI and have already been examined by F$^3$D (Alg. \ref{Algorithm_1}, lines 2, 29).
\end{enumerate}
with the above two sets, F$^3$D can determine whether there is a need to examine a specific voxel. Eventually, F$^3$D can guarantee the searching region required by frontier detection is non-repetitive, as the one shown in Fig. \ref{frontier_detection}(c).

In one specific FOV, we utilize a BFS-based strategy for extracting and clustering new frontiers. First, the voxels in the FOV will be examined in turn to check if it is a frontier voxel (Alg. \ref{Algorithm_1}, lines 13-31). Once a frontier voxel is detected, the BFS algorithm will be invoked for frontier extraction (Alg. \ref{Algorithm_1}, lines 24, 35). The BFS utilizes the new frontier voxel as a starting point to extract all frontier voxels that are belonging to one cluster based on the connectivity definition (Alg. \ref{Algorithm_2}, lines 4-17). Consequently, the detected frontier voxels are grouped naturally, and no more independent clustering operation is required. The new frontier is stored in \emph{newFrontier}, a database used to store the new frontier voxels temporarily (Alg. \ref{Algorithm_2}, lines 18). In general, the size of a frontier might be too large to make efficient decisions. Therefore, similar to \cite{TIM_Zhang_2022} and \cite{Zhou_2021}, we perform an additional frontier process operation to split large frontiers into small clusters and compute their centroids by the function \emph{Process\_Frontier()} (Alg. \ref{Algorithm_1}, line 40). This operation is conducted when all new frontiers are detected.

As shown in Fig. \ref{frontier_detection}(d), F$^3$D can detect frontiers completely as well as accurately with high efficiency, and it is not necessary to be executed after each scan. Therefore, the flexibility of the proposed frontier detector is satisfactory. In Section. \ref{sub_section:Theoretical-Analysis}, we will prove that F$^3$D is complete and sound, and we will also analyze the computational complexity of F$^3$D.

\subsection{Uniformly Deterministic Sampling based Incremental Road Map Construction}\label{sub_section:Global-Roadmap-Construction}
The road map is responsible for pathfinding in the exploration. Initially, the road map ${\cal R}$ = $\{p_0, \emptyset\}$, where $p_0$ is the home position of the MAV. In our case, ${\cal R}$ needs to extend when new regions are covered by the sensor. Generally, random sampling strategies will be employed to sample a large space of the environment to construct the road map. However, random sampling tends to result in an uneven distribution of the road map or needs a lot of effort to obtain a sequence of evenly distributed nodes. Based on our tests, to achieve a similar distribution of nodes, the random sampling strategy typically requires sampling and evaluating several times more nodes than the deterministic sampling strategy. To better estimate the utility of the exploration candidates, the road map is required to distribute in the free space of the environment comprehensively and evenly. This motivates us to propose an incremental road map construction algorithm that uniformly samples the new regions observed by the onboard sensor of the MAV.
\begin{figure}[!t]
	\centering
	\includegraphics[width=1.0\hsize]{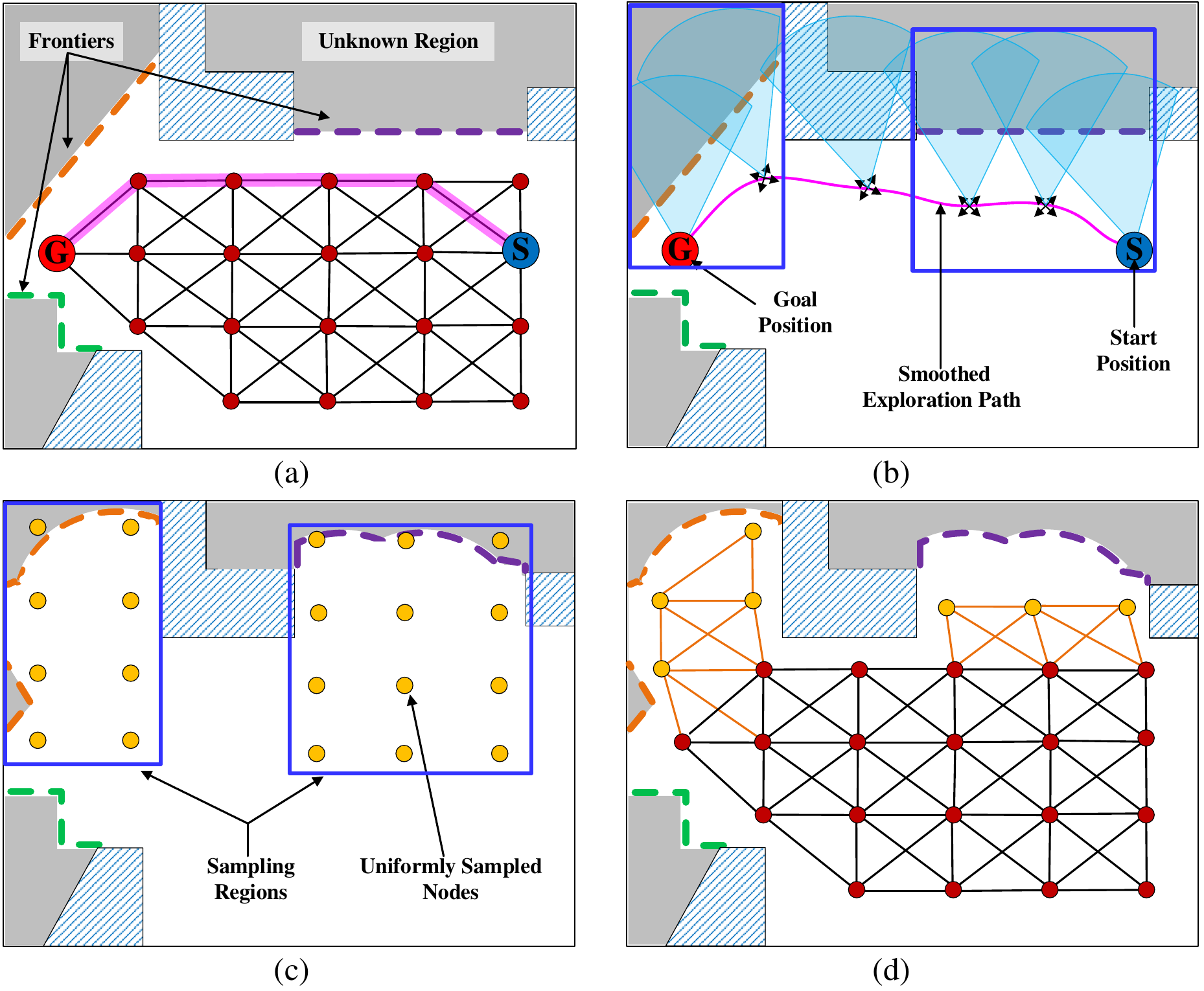}\vspace{-0.5cm}
	\caption{The 2-D schematic diagram of our uniformly deterministic sampling based incremental road map construction. (a) The exploration path can be found on the road map ${\cal R}$. (b) The MAV follows the exploration path while recording its FOVs. (c) All the FOVs that intersect with the same new frontier are extracted and merged into one AABB (i.e., the blue rectangles). In that region, uniform sampling is conducted. (d) The feasible new nodes (the yellow circles) and corresponding edges (the brown lines) are added to ${\cal R}$.}
	\label{roadmap_construction}
\end{figure}

As shown in Fig. \ref{roadmap_construction}, when the exploration path is determined, the MAV will follow this path to explore unknown regions of the environment. It is obvious that only the newly covered regions need to be sampled for sampling new road map nodes. If the sampling can be biased in these regions, the sampling efficiency will be improved. Fortunately, the FOVs of the onboard sensor are recorded in the frontier detection stage and we will utilize them to determine the necessary sampling region (Alg. \ref{Algorithm_3}, lines 3-12). Specifically, we first extract the FOVs that are intersecting with the new detected frontiers ${\cal F}_{new}$. In this process, the AABB of a FOV is computed to represent its effective region. Furthermore, if there are FOVs intersecting with each other, their AABBs will be merged into a larger one, e.g., the blue rectangles in Fig. \ref{roadmap_construction}(b) and (c) (Alg. \ref{Algorithm_3}, line 6). It is worth noting that our method can compute the sampling region adaptively instead of using a size-fixed and overestimated one.

\renewcommand{\algorithmicrequire}{\textbf{Input:}}
\renewcommand{\algorithmicensure}{\textbf{Output:}}
\begin{algorithm}[!t]
	\caption{Incremental Road Map Construction}
	\label{Algorithm_3}
	\begin{algorithmic}[1]
		\Require {$\mathcal{V}_{free}$, ${\cal S}_t$, ${\cal F}_{new}$} // ${\cal S}_t$ and ${\cal F}_{new}$ are from Alg. \ref{Algorithm_1}
		\Ensure	{${\cal R}$ // the road map}
		\State ${\cal S}_{temp}$ $\gets$ $\emptyset$;
		\State // Sampling regions determination
		\For{each \emph{s} $\in$ ${\cal S}_t$}
		\If {\emph{s} is intersecting with {\rm{f}}} //  $\forall\; {\rm{f}} \in {\cal F}_{new}$
		\If {\emph{s} is intersecting with $\hat{s}$} // $\forall\; \hat{s} \in {\cal S}_{temp}$
		\State $\hat{s}$ $\gets$ $\hat{s}$ $\cup$ \emph{s};
		\State  ${\cal S}_{temp}$ $\gets$ ${\cal S}_{temp}$ $\cup$ $\hat{s}$;
		\State  continue;
		\EndIf
		\State  ${\cal S}_{temp}$ $\gets$ ${\cal S}_{temp}$ $\cup$ \emph{s};
		\EndIf
		\EndFor
		\State // Road map construction
		\For{each \emph{s} $\in$ ${\cal S}_{temp}$}
		\State ${\cal P}$ $\gets$ \emph{Uniform\_Sampling}(\emph{s});
		\For{each feasible \emph{p} $\in$ ${\cal P}$}
		\If {${\left\| {p} - {p_i} \right\|} < {d_{\min }}$} // $\forall\; {p_i} \in {\cal R}$
		\State  continue;
		\EndIf
		\State ${\cal P}_{near}$ $\gets$ \emph{Near}(\emph{p}, ${\cal R}$, $d_{max}$);
		\For{each ${p_{near}}$ $\in$ ${\cal P}_{near}$}
		\If {\emph{No\_Collision}($p$, ${p_{near}}$, $\mathcal{V}_{free}$) and \emph{$||p - {p_{near}}|| < d_{max}$}}
		\State $E$ $\gets$ \emph{Connect}($p$, ${p_{near}}$);
		\State ${\cal R}$ $\gets$ ${\cal R}$ $\cup$ $\{p, E\}$;
		\EndIf
		\EndFor
		\EndFor
		\EndFor
	\end{algorithmic}
\end{algorithm}

In each sampling region, we utilize the Sukharev grid \cite{Janson_2018}, which aims to achieve the most uniform distribution possible over the sampling region, to generate a set of new nodes. Specifically, a Sukharev grid is generated with the given resolutions $\left[ l_x, l_y, l_z \right]$ in the $x$, $y$, $z$ directions to cover each sampling region. After that, the sampling region will be partitioned into separated cells, and the centroids of these cells form the set of new nodes, ${\cal P}$. If one new node $p_{new} \in {\cal P}$ meets the following two conditions, it will be considered as a feasible node for constructing the road map, i.e.
\begin{equation}\label{Eq_2}\vspace{-0.25cm}
  {p_{new} \in {V_{free}}}
\end{equation}
\begin{equation}\label{Eq_3}
  {{d_{\min }} \le D(p_{new},{p}) \le {d_{\max }}, \;\forall\; {p} \in {\cal R}}
\end{equation}
where $p$ is an existing node of ${\cal R}$ and $D(p_{new},{p})$ is the Euclidean distance between $p_{new}$ and $p$. $d_{\min}$ and $d_{\max}$ are two thresholds for restricting the edge length of the road map. The first condition (\ref{Eq_2}) requires the new node is lying within the free space. The second condition (\ref{Eq_3}) ensures that the road map is not too dense and distributed evenly. Furthermore, when adding $p_{new}$ to ${\cal R}$, the edge between $p_{new}$ to the node $p \in {\cal R}$, $E(p_{new}, p)$, is also added if it is collision-free and meeting the condition stated in (\ref{Eq_3}). The detail of the incremental road map construction process is shown in Alg. \ref{Algorithm_3}, lines 14-28. Further, to cope with dynamic or previously overlooked obstacles, we periodically check and prune the road map ${\cal R}$ for ensuring all the nodes and edges of ${\cal R}$ are collision-free.

After the road map is extended, we compute the nearest node $p_{nearest}$ for each frontier ${\rm{f}} \in {\cal F}_t$ and these nodes will be regarded as exploration candidates. Note that the candidate node is the one that can connect to its corresponding frontier without collision. With the road map, the path from the current position of the robot to the candidate as well as the corresponding motion cost can be queried efficiently. In what follows, we will show that one can achieve the global optimal decision-making based on the road map.

\begin{figure*}[!t]
	\centering
	\includegraphics[width=0.9\hsize]{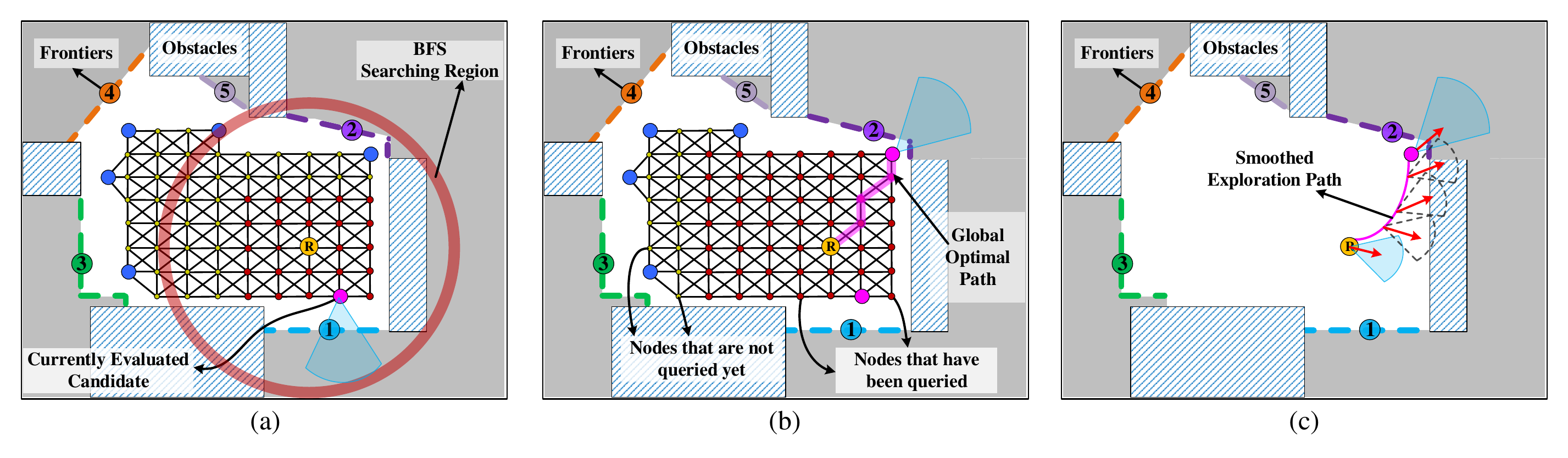}
	\caption{The 2-D illustration of our two-stage exploration path planning method. (a) When the Dijkstra algorithm finds a new exploration candidate (the purple solid circle), the new radius of the Dijkstra's searching region (the region inside the red ring) will be computed. (b) The Dijkstra algorithm continues searching outward and finds the next new exploration candidate (the purple solid circle at the top right corner). Obviously, this new candidate has the maximum exploration gain, and thus no further searching is necessary. The reason is that the remaining candidates can not achieve a higher utility than this new candidate, since they are guaranteed to have a higher motion cost than this new candidate. (c) The global optimal exploration path is obtained, and we perform a path smoothing technique to facilitate the MAV's tracking.}
	\label{informative_path_planning}
\end{figure*}
\subsection{Two-stage Exploration Path Planning}\label{sub_section:Local-Path-Planning}
The path $\gamma({\rm{x}}_r, {\rm{x}}_i)$ from the current state ${\rm{x}}_r$ of the MAV to any exploration candidate ${\rm{x}}_i$ can be queried efficiently on the road map ${\cal R}$ using a graph search algorithm, e.g., the Dijkstra algorithm. In addition, the utility of the candidate is typically computed by considering both the exploration gain and the motion cost \cite{survey_2012}. In this section, we propose a two-stage planner that computes the global optimal exploration path first and then improves its smoothness and time allocation.

\subsubsection{Global Optimal Exploration Path Planning}\label{sub_sub_section:Information-Gain-Evaluation}
In the first stage, we evaluate all exploration candidates by considering their exploration gains and motion costs simultaneously. The utility function of (\ref{Eq_1}) is formulated as follows
\begin{equation}\label{Eq_4}
  {\cal U}({\cal L}(\gamma),\;{\cal I}({\gamma})) = {\cal I}({\gamma({\rm{x}}_r, {\rm{x}}_i)}){e^{ - {\cal L}({\gamma({\rm{x}}_r, {\rm{x}}_i)})}}
\end{equation}
where the exploration gain ${\cal I}$ can be computed by
\begin{equation}\label{Eq_5}
I(\gamma ({{\rm{x}}_r},{{\rm{x}}_i})) = \left\{ {\begin{array}{*{20}{c}}
{{\rm{v}}Gain({{\rm{x}}_i})}\\
{\rm{or}}\;\;\;\;\;\\
{{{\rm{i}}Gain({\rm{x}}_i)}}
\end{array}} \right.
\end{equation}
where ${\rm{v}}Gain()$ is the voxel gain, ${\rm{i}}Gain()$ is the information gain. To maximize the exploration gain of each specific exploration candidate ${\rm{x}}_i := \{p_i, \xi _i\}$, we determine the yaw angle ${\xi _i}$ of the MAV at ${p_i}$ as the one maximizing ${\cal I}$ by using a yaw optimization method similar to \cite{Zhou_2021}.

\emph{Remark 1:} For the exploration candidate, one can evaluate its exploration gain by calculating either the voxel gain or information gain. The voxel gain \cite{Bircher_2018, Selin_2019}, ${\rm{v}}Gain()$, is the total number (or volume) of unknown voxels that can be observed at the exploration candidate. The information gain \cite{TASE_Chaoqun_2019, TASE_haiming_2019}, ${\rm{i}}Gain()$, is the mutual information that the robot can obtain at the exploration candidate. In this work, we employ ${\rm{v}}Gain()$ to compute the exploration gain of one candidate, as it is more efficient than computing the information gain ${\rm{i}}Gain()$.

The motion cost for the MAV following the exploration path is formulated as follows
\begin{equation}\label{Eq_6}
   {\cal L}({\gamma({\rm{x}}_r, {\rm{x}}_i)}) = {\lambda} dist({\gamma({\rm{x}}_r, {\rm{x}}_i)})
\end{equation}
where $dist()$ denotes the path length queried on the road map and ${\lambda} \geq 0$ is a tunable parameter.

To achieve the global optimal planning, it is intuitive to evaluate the utility of all candidates and then choose the best one as the exploration target. However, computing the exploration gain and motion cost for all candidates typically results in a high computational overhead, which may prevent the application of our method in complex and large 3-D environments. Therefore, we propose a lazy evaluation strategy that allows us to obtain the global optimal exploration path while only evaluating a subset of all candidates. As shown in Fig. \ref{informative_path_planning}, we utilize the Dijkstra algorithm to search paths between the robot's current state ${\rm{x}}_r$ and the exploration candidates on the road map. Recall that when the onboard sensor is determined, its FOV can be defined in advance. Therefore, the maximum attainable exploration gain at any candidate state ${\rm{x}}$ will not exceed a bounded constant value $I_{max}$. With this fact in mind, we can continually reduce the radius of the Dijkstra's searching region when evaluating candidates, thereby lowering the computational cost without compromising the global optimality. 

Given the currently evaluated candidate ${\rm{x}}_i$, we can compute its utility ${\cal U}_i$ by using (\ref{Eq_4}). Let us assume that a candidate ${\rm{x}}_j$ to be evaluated has the maximum attainable exploration gain $I_{max}$. Then, if we want ${\cal U}_j \geq {\cal U}_i$, the maximum motion cost of ${\rm{x}}_j$ should meet the following condition
\begin{equation}\label{Eq_7}
	{\cal L}_j^{max} \leq - \frac{1}{\lambda} \ln( \frac{{\cal U}_i}{I_{max}})
\end{equation}    
where we always have ${\cal U}_i =  {\cal I}_i {e^{ - {\cal L}_i}} \leq I_{max}$, which can guarantee ${\cal L}_j^{max}$ is valid, and ${\cal L}_j^{max}$ will be set as the new radius of the Dijkstra's searching region. 

At the beginning of each planning iteration, we first set the searching region as ${\cal L}^{max} = + \infty$. As shown in Fig. \ref{informative_path_planning}(a), every time the Dijkstra algorithm searches a new candidate ${\rm{x}}_{new}$, we first compute its utility ${\cal U}_{new}$ and the corresponding maximum motion cost ${\cal L}_{new}^{max}$. If ${\cal U}_{new}$ is larger than the utility of the previous best candidate ${\rm{x}}_{old}$, we replace ${\rm{x}}_{old}$ by ${\rm{x}}_{new}$. Otherwise, the Dijkstra algorithm will continue searching for the next candidate (see Fig. \ref{informative_path_planning}(b)). The above process is performed iteratively until the Dijkstra algorithm reaches the region boundary, and thus, stop the search. After that, we can obtain the global optimal exploration goal ${\rm{x}}_g$ as well as the shortest path towards it. Practically, the Dijkstra's searching region will be reduced dramatically after evaluating several candidates, resulting in a significant improvement of efficiency.

\subsubsection{Path Smoothing}\label{sub_sub_section:Exploration-Path-Planning}
In the second stage, we will refine the optimal exploration path that directly queried on the road map. It can be seen in Fig. \ref{informative_path_planning}(b) that the global optimal exploration path may be tortuous and hence not appropriate for the robot tracking. To improve the smoothness of the path, we utilize a path smoothing algorithm similar to \cite{Tianyi_2024} that can compute the smoothest path while ensuring its clearance for safe flight in narrow spaces. As shown in Fig. \ref{informative_path_planning}(c), the smoothed exploration path will be more appropriate for tracking. In addition, to further speed up the exploration process, a numerical integration-based time-optimal velocity planning algorithm \cite{kunz2012time} is employed to generate the optimal time allocation along the smoothed exploration path.

\subsection{Theoretical Analysis}\label{sub_section:Theoretical-Analysis}

\subsubsection{F$^3$D Completeness and Soundness}\label{sub_sub_section:FFFD-Completeness-and-Soundness}
Similar to \cite{IJRR2014}, we begin with a lemma that demonstrates F$^3$D can always identify all new frontiers correctly (i.e., F$^3$D will completely and correctly mark all new frontier voxels as frontier, which were not frontier voxels before current detection iteration). This lemma can help us prove F$^3$D is complete and sound.
\begin{lemma}
	Suppose $v_f$ is a new frontier voxel at iteration $t$, which was not marked as a frontier voxel at any iteration $\tau$, where $\tau < t$. Then, given a set of observations $O^t_{t-1}$ that accumulated during iteration $t-1$ and iteration $t$, F$^3$D will mark $v_f$ as a frontier voxel.
\end{lemma}
\begin{proof}
	Let $O^t_{t-1}$ be the set of sensor observations accumulated during the period of two iterations $t$ and $t-1$, and let the voxels covered by $O^t_{t-1}$ be denoted by $\mathcal{V}(O^t_{t-1})$, i.e., the space that consists of $O^t_{t-1}$. At each iteration, only voxels in $\mathcal{V}(O^t_{t-1})$ will change their states. Furthermore, according to the frontier definition, only free voxels in $\mathcal{V}(O^t_{t-1})$ have the potential to be a frontier. F$^3$D will examine every free voxel in $\mathcal{V}(O^t_{t-1})$ and mark it as a frontier if it is a valid frontier voxel (Alg. \ref{Algorithm_1}, lines 13-31).
\end{proof}

Based on Lemma 1, we can draw Theorem 1 as follows.
\begin{theorem}
	Let $v_f$ be a valid frontier voxel at iteration $t$. Then F$^3$D will mark $v_f$ as a frontier voxel given the set of observations $\{O^1_0, ..., O^t_{t-1}\}$ that accumulated up to now.
\end{theorem}
\begin{proof}
In order to prove the completeness of F$^3$D, there are two cases that need to be examined:	

{\bf{Case 1.} $v_f$ is a new frontier voxel at iteration $t$.} This case can be proved directly by Lemma 1.

{\bf{Case 2.} $v_f$ was a new frontier voxel at iteration $\tau$, where $\tau < t$.} Let $\tau$ be the earliest iteration in which $v_f$ was as a frontier. Based on Lemma 1, we know that $v_f$ was marked as a frontier at that time. If $v_f$ is not intersecting with the FOV of the onboard sensor recorded during the interval of iterations $t-1$ and $t$, which means $v_f$ has no chance to change its states because it was not covered yet by the sensor of the robot. If $v_f$ is intersecting with one of the FOVs, it will be reexamined by F$^3$D to confirm whether it is a frontier or not (Alg. \ref{Algorithm_1}, lines 3-11 and 33-39). Therefore, if $v_f$ is still be a frontier voxel at iteration $t$, it will be correctly marked by F$^3$D. Furthermore, F$^3$D will always maintain knowledge of the valid old frontiers from the time that they are identified. Consequently, $v_f$ must be a frontier voxel that is maintained by F$^3$D at iteration $t$.

The above two cases show that F$^3$D will identify $v_f$ to be a valid frontier voxel at iteration $t$. Consequently, we can say Theorem 1 is true for any valid frontier voxel at iteration $t$, which follows that F$^3$D is complete.
\end{proof}
In order to prove the soundness of F$^3$D, we must demonstrate that there does not exist a case where F$^3$D identifies a voxel as a frontier voxel when it is not.
\begin{theorem}
	Let $v$ be an arbitrary voxel at iteration $t$, which is not a frontier voxel. Then F$^3$D will not mark $v$ as a frontier voxel given a set of observations $\{O^1_0, ..., O^t_{t-1}\}$.
\end{theorem}
\begin{proof}
	Assuming that $v$ is an arbitrary voxel at iteration $t$, i.e., $v$ is not a frontier voxel at iteration $t$. Then, according to our definition, $v$ is either not a free voxel or $v$ is a free voxel but all of its 6-connected neighbors are not unknown voxels. We have two cases that need to be considered:

{\bf{Case 1.} $v$ was a frontier voxel at iteration $t-1$.} If $v$ was a frontier voxel at iteration $t-1$, then it could be covered by the sensor (i.e., it might intersecting with one of the FOVs of the sensor). Therefore, $v$ will be rechecked and removed from the frontier database (Alg. \ref{Algorithm_1}, lines lines 3-11, 33-39). However, there may be a rare case, i.e., when $v$ is close to but outside the FOVs of the sensor and the sensor covered its only unknown neighbor. In this special case, $v$ will not be a frontier voxel anymore. Fortunately, F$^3$D will check the 6-connected neighbors of each voxel inside the FOVs of the sensor, and thus this special case can be handled (Alg. \ref{Algorithm_1}, lines 18-22).

{\bf{Case 2.} $v$ was not a frontier voxel at iteration $t-1$.} If $v$ is not covered by $O^t_{t-1}$, then F$^3$D will not scan it and therefore will not mark it as a frontier voxel. If $v$ is a voxel that is covered by $O^t_{t-1}$, it will be checked and will not be marked as a frontier (Alg. \ref{Algorithm_1}, lines 13-31).
\end{proof}

\begin{figure}[!t]
	\centering
	\includegraphics[width=1.0\hsize]{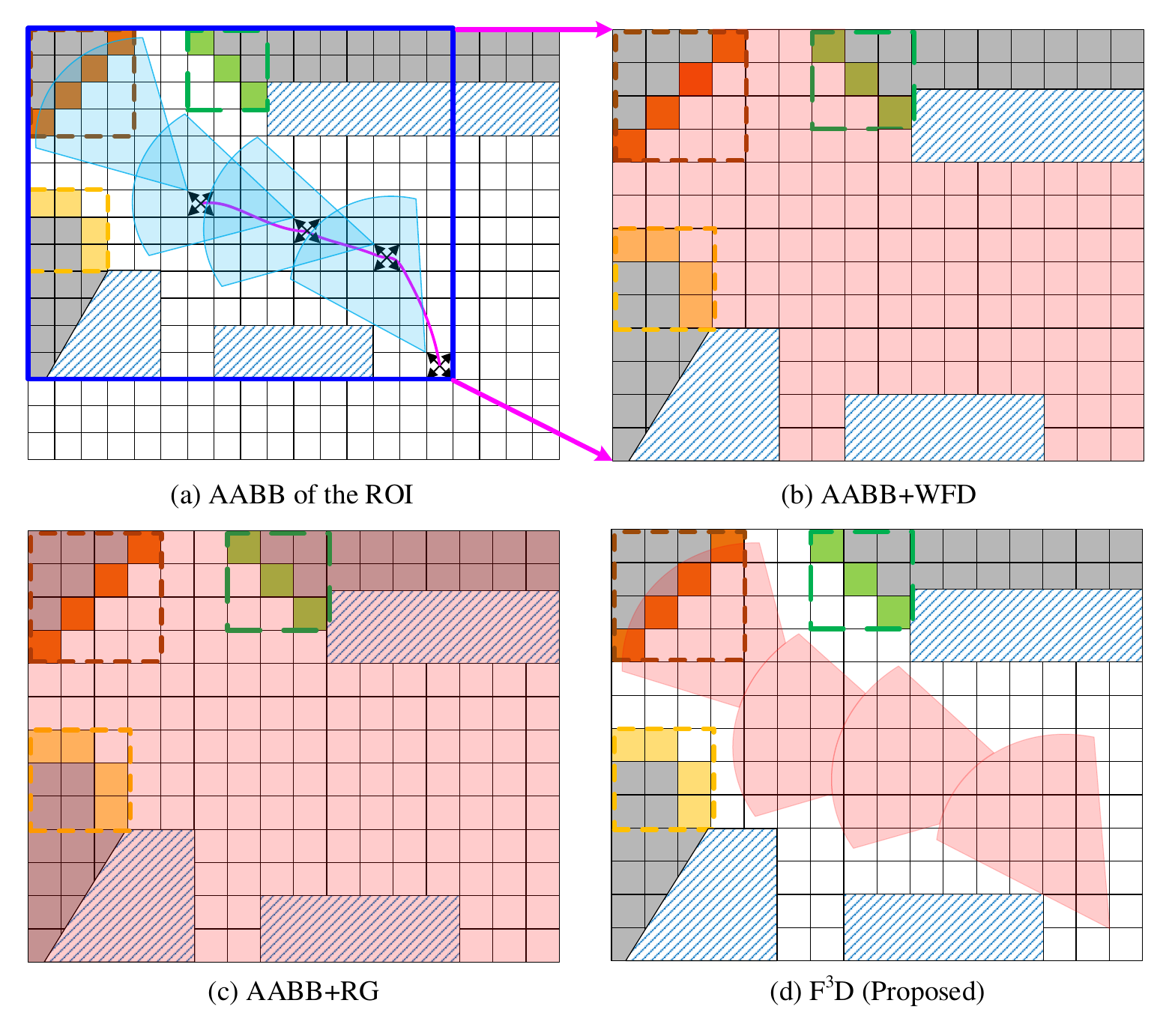} \vspace{-0.5cm}
	\caption{A comparison of different frontier detection strategies. (a) Illustration of the AABB (the blue rectangle) of the sensor covered region recorded in the duration of two iterations. (b) The AABB+WFD method \cite{IJRR2014} scans all free voxels in the AABB. (c) The AABB+RG method \cite{Zhou_2021} scans all voxels in the AABB. (d) F$^3$D only scans voxels in the FOVs of the sensor.}
    \label{frontier_detection_comparison}
\end{figure}

\subsubsection{Computational Complexity}\label{sub_sub_section:Computational-Complexity}
For ease of computation, we use a AABB to estimate the region covered by each FOV of the sensor. Let $n \in \mathbb{Z}^+$ be the number of the recorded FOVs for frontier detection and $m \in \mathbb{Z}^+$ be the number of frontiers in $\mathcal{F}_{t-1}$. Since we use the standard C++ list template as the container to store frontiers, the time for removing one frontier from $\mathcal{F}_{t-1}$ is $\mathcal{O}$ ($1$). Therefore, the lines 3-11 of Alg. \ref{Algorithm_1} runs in $\mathcal{O}$ ($mn$) time. After that, F$^3$D begins detecting new frontiers (Alg. \ref{Algorithm_1}, lines 13-26). Within each call to F$^3$D, it scans every voxel in the recorded FOVs for frontier voxels. Let $k \in \mathbb{Z}^+$ be the total number of voxels in the ROI and the time complexity for detecting new frontiers is $\mathcal{O}$ ($6k$) due to our connectivity definition. As can be seen in Fig. \ref{frontier_detection_comparison}, the voxels that need to be checked for frontier detection by \cite{Zhou_2021} and \cite{IJRR2014} are significantly more than ours. Intuitively, AABB+WFD scans less voxels than AABB+RG and its time complexity should be supposed to less than AABB+RG. However, according to the experimental results (see Section \ref{section:Experiments}), the conclusion is opposite. This is because the AABB+WFD method maintains four container called \emph{Map-Open-List}, \emph{Map-Close-List}, \emph{Frontier-Open-List}, and \emph{Frontier-Close-List} to achieve non-repetitive detection \cite{IJRR2014}. In contrast, our method only maintains two such container and the magnitude of data is more less than AABB+WFD, i.e., only overlapping areas need to be handled. Then, if a voxel in the ROI is identified as a new frontier voxel, the BFS will be invoked to extract all voxels that belonging to one frontier (Alg. \ref{Algorithm_2}). According to \cite{BFS2001}, we know that the complexity is linear in size of the number of new frontier voxels, i.e., $\mathcal{O}$ ($N(new\--frontier\--voxels)$).Also, let $h \in \mathbb{Z}^+$ be the voxels in \emph{deleteSet}, the lines 28-34 of Alg. \ref{Algorithm_1} takes $\mathcal{O}$ ($6h$) time. Finally, the time complexity of F$^3$D is $\mathcal{O}$ $(mn + 6(k + h) + N(new-frontier-voxels))$, which is close to the linear time complexity.

\section{Evaluation Results}\label{section:Experiments}

\subsection{Implementation Details}\label{sub_section:Implementation-Details}
In order to evaluate our method (FSMP) thoroughly, we conduct both simulation and real-world experiments in the context of an MAV. Fig. \ref{robot_platform} shows the robot platforms used in the simulation and real world, respectively. 

For the simulations, all algorithms are implemented by C++ on an OMEN9 SLIM laptop that runs Linux Ubuntu 20.04 LTS operation system with Intel Core i9-13900HX CPU at 5.4 GHz, 16-GB memory. We choose Gazebo \cite{gazebo} as the simulation engine since it provides realistic environments and robot models. The Firefly MAV provided by RotorS \cite{Furrer2016} (see Fig. \ref{robot_platform}(a)) is spawned in Gazebo to explore unknown 3-D environments. In our setup, the specifications of the VI-Sensor mounted on Firefly MAV have a sensing range of $\left[ {0.5,5} \right]$$m$ and a field of view of $\left[ 110, 90 \right]^\circ$ in horizontal and vertical directions. Unless specified, otherwise, the rest of parameters of our algorithm used in the simulation are listed in Table \ref{table_params}.

For the real-world experiments, we run all algorithms on an Intel Core i7-1260P CPU. The custom-built MAV used in the real-world experiment is shown in Fig. \ref{robot_platform}(b). We localize the MAV by using a Livox Mid 360 LiDAR and build the volumetric map of the environment using data captured by the Intel Realsense D435i depth camera. The specifications of D435i have a sensing range of $\left[ {0.3,5} \right]$$m$ and a field of view of $\left[ 87,58 \right]^\circ$ in horizontal and vertical directions. The other parameters of the exploration algorithm used in the experiment are the same as in the simulations.

\begin{figure}[!t]
	\centering
	\includegraphics[width=0.95\hsize]{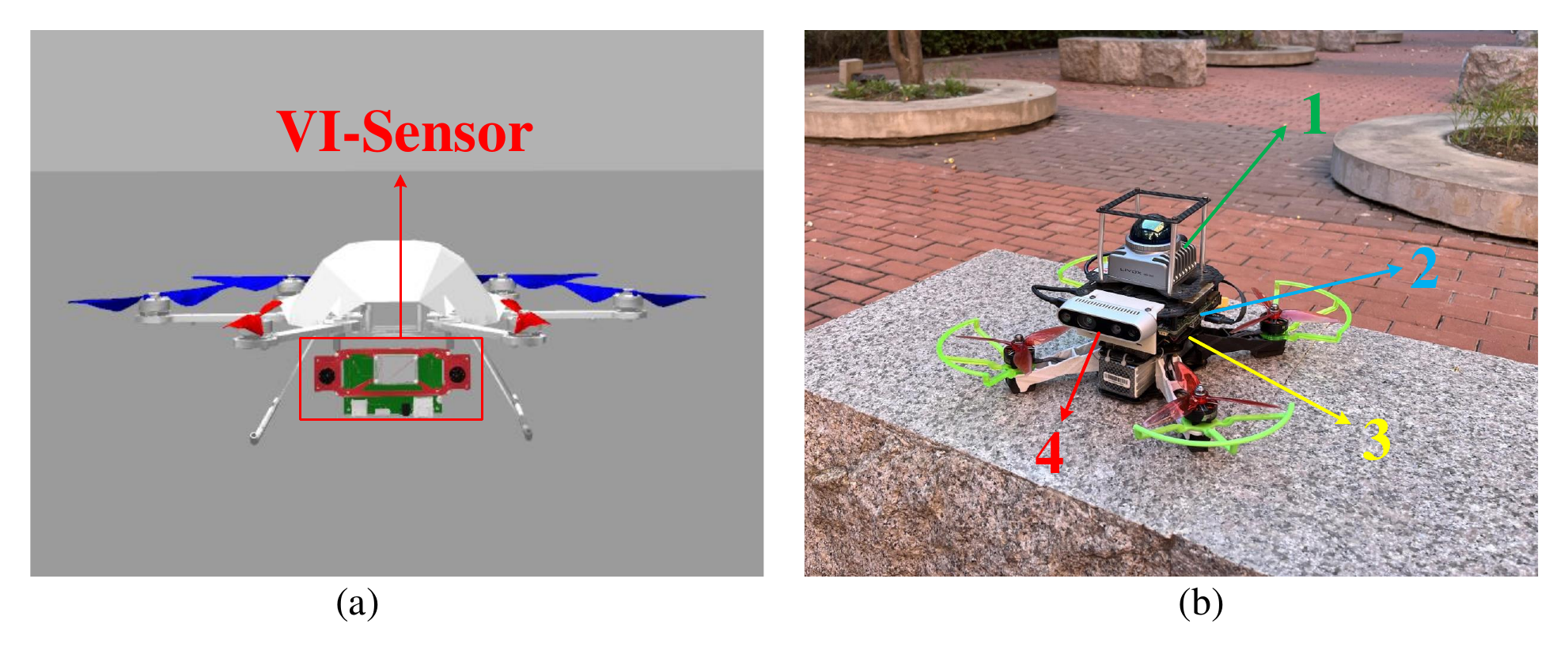}\vspace{-0.1cm}
	\caption{The robot platforms that used in the simulation and real-world experiments. (a) The simulated MAV, Firefly, is equipped with a forward-looking depth camera (i.e., VI-Sensor). (b) The real MAV is equipped with 1) a Livox Mid 360 LiDAR, 2) an onboard computer, 3) a PIXHAWK autopilot, and 4) a forward-looking depth camera.}
	\label{robot_platform}
\end{figure}

\begin{table}[!t]\centering
\caption{Parameters used in the simulation and experiment} \label{table_params}
\begin{tabular}{|c|c||c|c|}
\hline
\textbf{Parameter} & \textbf{Value} & \textbf{Parameter }& \textbf{Value}\\
\hline
${d_{min}}$ & 0.5 & $l_x$ & $0.8$ \\
\hline
${d_{max}}$ & 1.5 & $l_y$ & $0.8$ \\
\hline
${\lambda}$ & 0.5 & $l_z$ & $0.8$\\
\hline
\end{tabular}
\end{table}

\subsection{Simulation Study}\label{sub_section:Simulation-Experiments}
\begin{figure*}[!t]
	\centering
	\includegraphics[width=0.95\hsize]{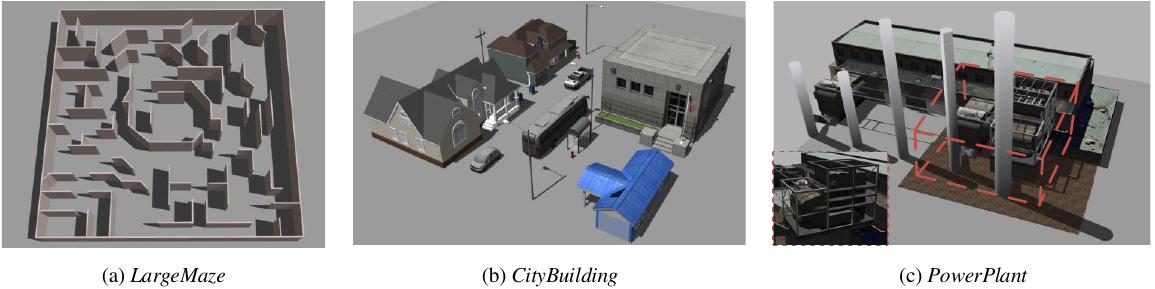}
	\caption{Three complex and large 3-D environments of increasing dimension and complexity are used in the simulation study. (a) The \emph{LargeMaze} environment is with the size of 40 $\times$ 40 $\times$ 3 $m^3$. (b) The \emph{CityBuilding} environment is with the size of 40 $\times$ 40 $\times$ 10 $m^3$. (c) The \emph{PowerPlant} environment is cropped to smaller dimensions (the part enclosed by the dotted box) and the size of the region of interest is 31 $\times$ 31 $\times$ 26 $m^3$.}
    \label{simulation_environments}
\end{figure*}
In this section, we evaluate our proposed planner in three complex and large 3-D synthetic environments. As shown in Fig. \ref{simulation_environments}, the environments involved in the simulation study are named as \emph{LargeMaze} (see Fig. \ref{simulation_environments}(a)), \emph{CityBuilding} (see Fig. \ref{simulation_environments}(b)) and \emph{PowerPlant} (see Fig. \ref{simulation_environments}(c)) respectively. 

First of all, to verify the efficiency of the proposed frontier detector, F$^3$D, we compare it with AABB+RG \cite{Zhou_2021} and AABB+WFD \cite{IJRR2014} in the simulation trials. The average values and standard deviations of the computation time of different frontier detectors are shown in Table \ref{table_frontier_detector_comparison}. The results indicate that F$^3$D achieves shorter time for frontier detection and smaller time variance in all environments. More specifically, F$^3$D is one order of magnitude faster than AABB+WFD and several times faster than AABB+RG. In addition, the performance of F$^3$D is consistently satisfactory in all environments, which indicates F$^3$D has a good scalability in different environments. It is worth noting that all the above frontier detectors are complete and sound, therefore, the accuracy of these detected frontiers is identical according to our tests.
\begin{table}[!t]\centering
\caption{Computation time of different frontier detectors} \label{table_frontier_detector_comparison}
\begin{tabular}{|c|clclcl|}
\hline
\multirow{3}{*}{\textbf{Scenario}} & \multicolumn{6}{c|}{\textbf{Computation Time} (ms)}                                                                                                                               \\ \cline{2-7}
                       & \multicolumn{2}{c|}{AABB+RG{\cite{Zhou_2021}}}                   & \multicolumn{2}{c|}{AABB+WFD{\cite{IJRR2014}}}                    & \multicolumn{2}{c|}{F$^3$D (Ours)}                               \\ \cline{2-7}
                       & \multicolumn{1}{c|}{Avg}   & \multicolumn{1}{c|}{Std}   & \multicolumn{1}{c|}{Avg}   & \multicolumn{1}{c|}{Std}  & \multicolumn{1}{c|}{Avg}           & \multicolumn{1}{c|}{Std} \\ \hline
\emph{LargeMaze}                   & \multicolumn{1}{c|}{15.5} & \multicolumn{1}{c|}{11.7} & \multicolumn{1}{c|}{53.1}  & \multicolumn{1}{c|}{51.6} & \multicolumn{1}{c|}{\textbf{4.3}} & \textbf{6.5}            \\ \hline
\emph{CityBuilding}                  & \multicolumn{1}{c|}{87.2} & \multicolumn{1}{c|}{54.7} & \multicolumn{1}{c|}{313.8}  & \multicolumn{1}{c|}{254.7} & \multicolumn{1}{c|}{\textbf{38.9}} & \textbf{35.6}            \\ \hline
\emph{PowerPlant}                  & \multicolumn{1}{c|}{55.3} & \multicolumn{1}{c|}{42.6} & \multicolumn{1}{c|}{155.0} & \multicolumn{1}{c|}{168.9} & \multicolumn{1}{c|}{\textbf{21.3}} & \textbf{35.0}            \\ \hline
\end{tabular}
\end{table}

We also compare our planner, FSMP, with three existing exploration methods which are tailored for MAVs equipped with depth-cameras. All the methods are implemented using their open-source code adapted to specific simulated environments.

\begin{itemize}
    \item \emph{IPP} \cite{Schmid_2020}: A sampling-based method that consistently span a single RRT* tree in the known free regions to achieve the global optimal exploration of the environment.
    
    \item \emph{FUEL} \cite{Zhou_2021}: A frontier-based method that computes the optimal sequence to visit all frontiers by addressing TSP. After that, it computes the minimum-time trajectory for fast exploration.

    \item \emph{FSample} \cite{Respall_2021}: A hybrid method that combines the Next-Best-View sampling and frontier-based strategies in a unified framework to achieve global exploration.
\end{itemize}

\emph{Remark 2:} It is worth noting that IPP is a pure sampling-based method, FUEL is a pure frontier-based method, and FSample is a hybrid method that combines both the sampling-based and frontier-based ideas. The reason we choose these three methods to verify our method is that they are very popular algorithms and each of them can be approximately viewed as a ablation case of our planner. 

Note that, in order to achieve a statistical comparison, we run all methods 5 times in each environment using the same initial configurations. Additionally, we record the statistics following each run and summarize all the comparative results in Table \ref{table_results_statistic}. These results include the {{Exploration Time}} when the explored volume reaches 90\% of the total space, the {{Traveling Distance}} when the explored volume reaches 90\% of the total space, and the {{Max Explored Volume}} when the time limitation (set as three times of our method) is reached.

\begin{figure}[!t]
	\begin{minipage}[!htp]{\linewidth}
		\centering
		\subfigure[Mapping Result and Executed Trajectories]{
			\label{LM_a}
			\includegraphics[width=0.9\hsize]{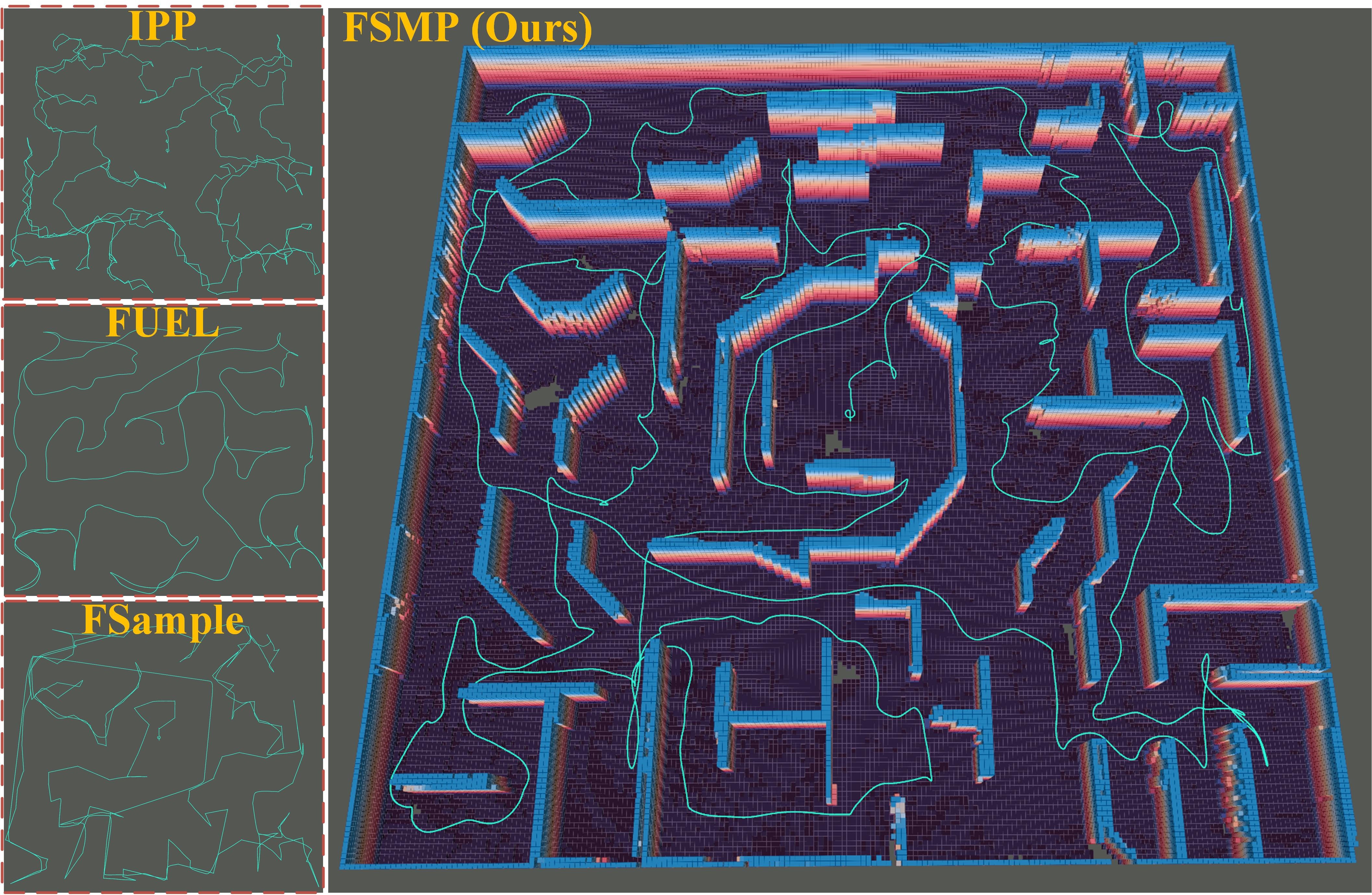}}
		\subfigure[Exploration Progresses]{
			\label{LM_b}
			\includegraphics[width=0.9\hsize]{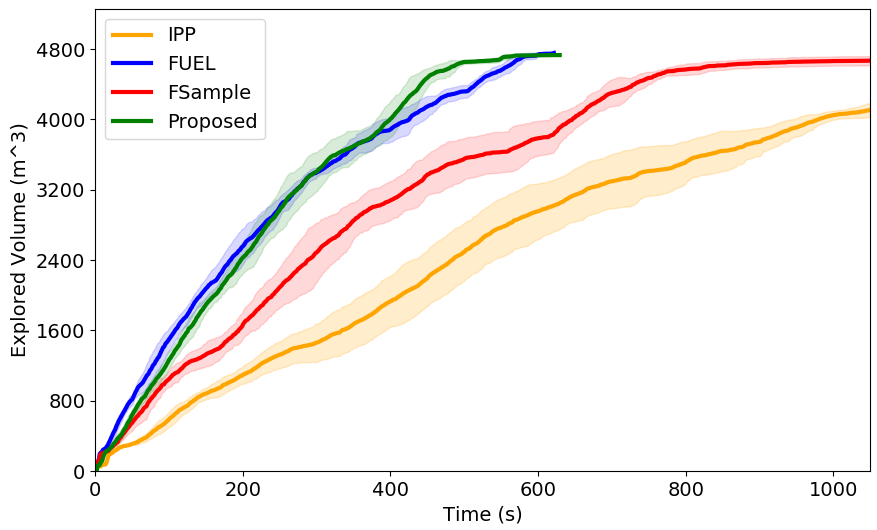}}
		\caption{The exploration results of the \emph{LargeMaze} environment. (a) shows the volumetric mapping result of our method and the executed trajectories of IPP, FUEL, FSample, and the proposed method from their representative runs. (b) shows the exploration progress of these four methods.}
		\label{LM_statistic}
	\end{minipage}
\end{figure}
The first trial uses the MAV to explore the \emph{LargeMaze} environment. This scenario is typically used to evaluate exploration methods and we also use it to benchmark the performance of our method. In this trial, we set the map resolution to $0.2$ $m$. The maximum linear velocity, maximum linear acceleration and maximum angular velocity of the MAV are set to 1.0 $m/s$, 1.0 $m/{s^2}$ and 1.0 $rad/s$, respectively. Fig. \ref{LM_statistic}(a) shows the volumetric mapping result of our method and the executed trajectories of each method. Table \ref{table_results_statistic} gives the statistics of these methods in the \emph{LargeMaze} environment. Our method can complete 90\% exploration of the environment using the shortest traveling distance (363.4$m$) and flight time (429.3$s$) on average. As shown in Table \ref{table_results_statistic}, our method can achieve 98.5\% coverage of the environment on average when it stops to explore. Fig. \ref{LM_statistic}(b) shows the average explored volume over the time of different methods. Our method improves the efficiency by more than 9.3\% for the first 90\% coverage, compared to all other methods. It is worth noting that FUEL, which is known as the state-of-the-art in the exploration field, demonstrated performance comparable to our method in this trial. The reason is that, in this scenario, the number of frontiers is relatively small, and thus the cost of solving TSP is low. The computational time used for each planning iteration of FUEL can be found in Table \ref{table_runtime_comparison}.

\begin{table*}[!t]\centering
	\caption{Statistical results of different methods with the same simulation configurations} \label{table_results_statistic}
	\begin{tabular}{|c|c|cccc|cccc|cccc|}
		\hline
		\multirow{2}{*}{\textbf{Scenario}} & \multirow{2}{*}{\textbf{Method}} & \multicolumn{4}{c|}{\textbf{Exploration Time} (s)} & \multicolumn{4}{c|}{\textbf{Traveling Distance} (m)}  
		& \multicolumn{4}{c|}{\textbf{Max Explored Volume} (m$^3$)}  \\ \cline{3-14} &                                  
		& \multicolumn{1}{c|}{\textbf{Avg}}    & \multicolumn{1}{c|}{\textbf{Std}}   & \multicolumn{1}{c|}{\textbf{Max}}    & \textbf{Min}    
		& \multicolumn{1}{c|}{\textbf{Avg}}    & \multicolumn{1}{c|}{\textbf{Std}}  & \multicolumn{1}{c|}{\textbf{Max}}    & \textbf{Min}    
		& \multicolumn{1}{c|}{\textbf{Avg}}    & \multicolumn{1}{c|}{\textbf{Std}}   & \multicolumn{1}{c|}{\textbf{Max}}    & \textbf{Min}    \\ \hline
		\multirow{4}{*}{\emph{LargeMaze}}      
		& IPP                               & \multicolumn{1}{c|}{1220.5}          & \multicolumn{1}{c|}{211.8}           & \multicolumn{1}{c|}{1420.5}          & 1167.0
		& \multicolumn{1}{c|}{513.4}           & \multicolumn{1}{c|}{43.4}            & \multicolumn{1}{c|}{578.5}           & 470.6          
		& \multicolumn{1}{c|}{4533.6}          & \multicolumn{1}{c|}{113.5}           & \multicolumn{1}{c|}{4651.3}          & 4372.0          \\ \cline{2-14}
		
		& FUEL                              & \multicolumn{1}{c|}{469.5}               & \multicolumn{1}{c|}{39.1}          & \multicolumn{1}{c|}{517.5}            & 460.0          
		& \multicolumn{1}{c|}{406.6}               & \multicolumn{1}{c|}{25.4}          & \multicolumn{1}{c|}{436.4}            & 369.5 
		& \multicolumn{1}{c|}{\textbf{4755.8}}     & \multicolumn{1}{c|}{15.4}          & \multicolumn{1}{c|}{\textbf{4766.9}}  & \textbf{4728.6}          \\ \cline{2-14}
		
		& FSample                           & \multicolumn{1}{c|}{692.0}          & \multicolumn{1}{c|}{51.9}             & \multicolumn{1}{c|}{735.5}         & 666.0          
		& \multicolumn{1}{c|}{461.3}          & \multicolumn{1}{c|}{\textbf{21.5}}    & \multicolumn{1}{c|}{486.2}         & 443.9         
		& \multicolumn{1}{c|}{4703.7}         & \multicolumn{1}{c|}{34.6}             & \multicolumn{1}{c|}{4750.2}        & 4672.9           \\ \cline{2-14}
		
		& Ours                                   & \multicolumn{1}{c|}{\textbf{429.3}}    & \multicolumn{1}{c|}{\textbf{18.6}}    & \multicolumn{1}{c|}{\textbf{451.1}}  & \textbf{398.6} 
		& \multicolumn{1}{c|}{\textbf{363.4}}    & \multicolumn{1}{c|}{22.6}             & \multicolumn{1}{c|}{\textbf{386.5}}  & \textbf{329.3}          
		& \multicolumn{1}{c|}{4731.3}            & \multicolumn{1}{c|}{\textbf{5.5}}     & \multicolumn{1}{c|}{4741.0}          & 4728.1  \\ \hline
		
		\multirow{4}{*}{\emph{CityBuilding}}         
		& IPP                               & \multicolumn{1}{c|}{-}                & \multicolumn{1}{c|}{-}           & \multicolumn{1}{c|}{-}          & -          
		& \multicolumn{1}{c|}{-}                & \multicolumn{1}{c|}{-}           & \multicolumn{1}{c|}{-}          & -        
		& \multicolumn{1}{c|}{11153.7}          & \multicolumn{1}{c|}{214.4}       & \multicolumn{1}{c|}{11366.6}    & 10856.8          \\ \cline{2-14}
		
		& FUEL                              & \multicolumn{1}{c|}{811.0}             & \multicolumn{1}{c|}{79.5}             & \multicolumn{1}{c|}{856.0}          & 775.5         
		& \multicolumn{1}{c|}{907.6}             & \multicolumn{1}{c|}{\textbf{51.9}}    & \multicolumn{1}{c|}{958.6}          & 848.7
		& \multicolumn{1}{c|}{12356.7}           & \multicolumn{1}{c|}{166.0}            & \multicolumn{1}{c|}{12562.2}        & 12188.3          \\ \cline{2-14}
		
		& FSample                           & \multicolumn{1}{c|}{1114.5}            & \multicolumn{1}{c|}{323.7}            & \multicolumn{1}{c|}{1286.5}            & 876.5          
		& \multicolumn{1}{c|}{882.6}             & \multicolumn{1}{c|}{109.6}            & \multicolumn{1}{c|}{1010.6}            & 728.5          
		& \multicolumn{1}{c|}{12474.3}           & \multicolumn{1}{c|}{272.1}            & \multicolumn{1}{c|}{12752.9}           & 12099.9          \\ \cline{2-14}
		
		& Ours                                    & \multicolumn{1}{c|}{\textbf{513.7}}     & \multicolumn{1}{c|}{\textbf{52.1}}   & \multicolumn{1}{c|}{\textbf{594.3}}    & \textbf{460.5} 
		& \multicolumn{1}{c|}{\textbf{609.2}}     & \multicolumn{1}{c|}{90.4}            & \multicolumn{1}{c|}{\textbf{733.6}}    & \textbf{535.2}        
		& \multicolumn{1}{c|}{\textbf{13349.8}}   & \multicolumn{1}{c|}{\textbf{119.0}}  & \multicolumn{1}{c|}{\textbf{13494.9}}  & \textbf{13167.2}  \\ \hline
		
		\multirow{4}{*}{\emph{PowerPlant}}           
		& IPP                               & \multicolumn{1}{c|}{-}          & \multicolumn{1}{c|}{-}          & \multicolumn{1}{c|}{-}          & -          
		& \multicolumn{1}{c|}{-}          & \multicolumn{1}{c|}{-}          & \multicolumn{1}{c|}{-}          & -          
		& \multicolumn{1}{c|}{16449.4}    & \multicolumn{1}{c|}{449.8}      & \multicolumn{1}{c|}{17161.5}    & 16065.3          \\ \cline{2-14}
		
		& FUEL                              & \multicolumn{1}{c|}{1539.8}          & \multicolumn{1}{c|}{130.5}          & \multicolumn{1}{c|}{1769.0}          & 1403.5          
		& \multicolumn{1}{c|}{1295.6}          & \multicolumn{1}{c|}{819.2}          & \multicolumn{1}{c|}{1376.3}          & 1219.2
		& \multicolumn{1}{c|}{19707.8}    & \multicolumn{1}{c|}{228.3}      & \multicolumn{1}{c|}{20077.7}    & 19518.3          \\ \cline{2-14}
		
		& FSample                           & \multicolumn{1}{c|}{1740.0}          & \multicolumn{1}{c|}{138.3}          & \multicolumn{1}{c|}{1809.5}          & 1626.0          
		& \multicolumn{1}{c|}{1509.1}          & \multicolumn{1}{c|}{66.7}           & \multicolumn{1}{c|}{1564.9}          & 1399.4          
		& \multicolumn{1}{c|}{19827.4}         & \multicolumn{1}{c|}{352.7}          & \multicolumn{1}{c|}{20280.3}         & 19466.7 \\ \cline{2-14}
		
		& Ours                                     & \multicolumn{1}{c|}{\textbf{686.1}}      & \multicolumn{1}{c|}{\textbf{16.8}}    & \multicolumn{1}{c|}{\textbf{709.7}}     & \textbf{662.2} 
		& \multicolumn{1}{c|}{\textbf{791.2}}      & \multicolumn{1}{c|}{\textbf{32.1}}    & \multicolumn{1}{c|}{\textbf{827.4}}     & \textbf{748.6}          
		& \multicolumn{1}{c|}{\textbf{21093.9}}    & \multicolumn{1}{c|}{\textbf{69.9}}    & \multicolumn{1}{c|}{\textbf{21169.4}}   & \textbf{20984.9}         \\ \hline
	\end{tabular}
	\footnotesize{${-}$: denotes a method can not achieve 90\% coverage of the environment.\;\;\;\;\;\;\;\;\;\;\;\;\;\;\;\;\;\;\;\;\;\;\;\;\;\;\;\;\;\;\;\;\;\;\;\;\;\;\;\;\;\;\;\;\;\;\;\;\;\;\;\;\;\;\;\;\;\;\;\;\;\;\;\;\;\;\;\;\;\;\;\;\;\;\;\;\;\;\;\;\;\;\;\;\;\;\;\;\;\;\;\;\;\;\;\;\;\;\;\;\;\;\;\;\;\;\;\;\;\;\;\;\;\;\;\;\;}
\end{table*}

\begin{figure}[!htp]
	\begin{minipage}[!htbp]{\linewidth}
		\centering
		\subfigure[Mapping Results and Executed Trajectories]{
			\label{CB_a}
			\includegraphics[width=0.95\hsize]{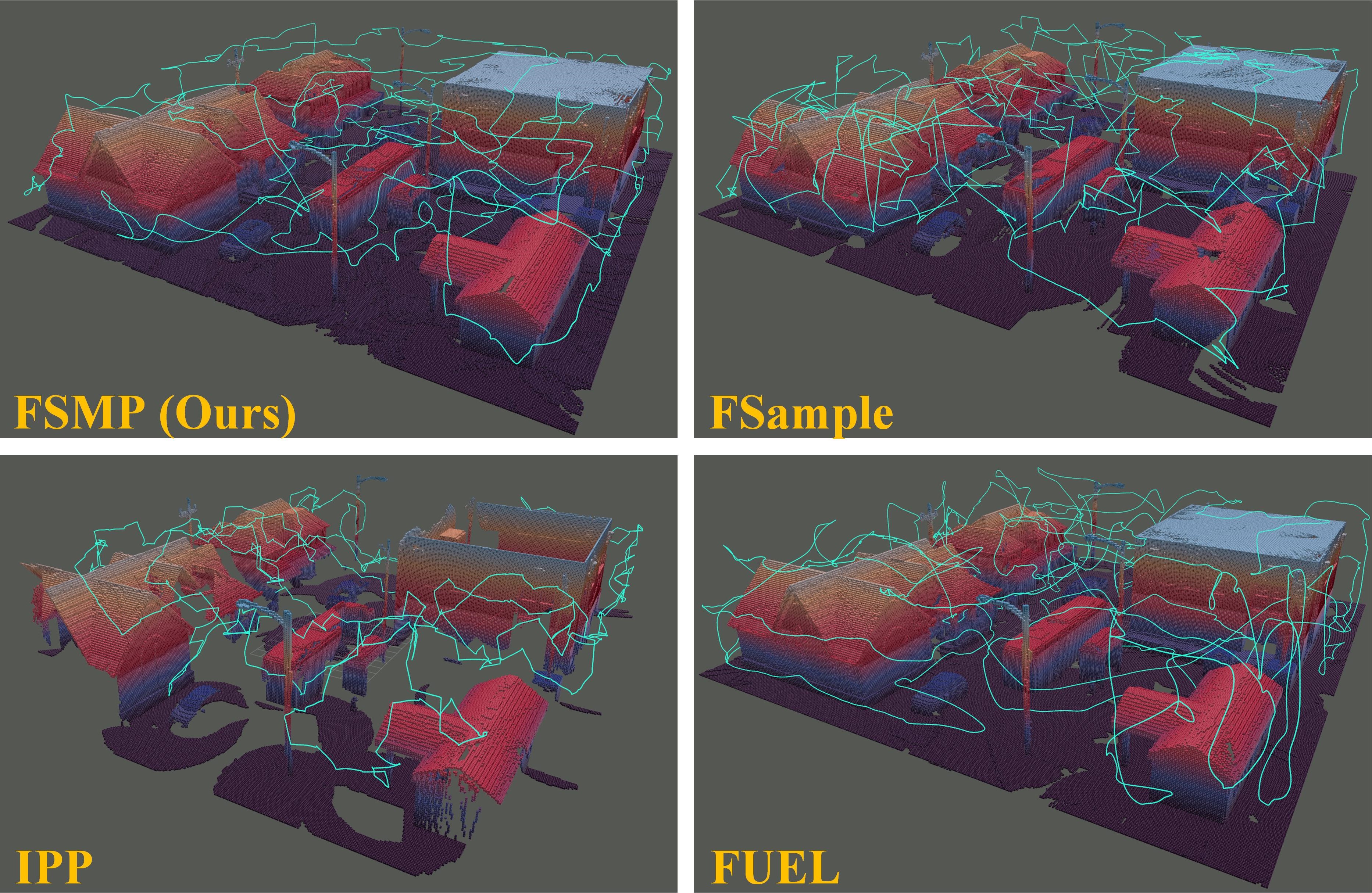}}
		\subfigure[Exploration Progresses]{
			\label{CB_b}
			\includegraphics[width=0.9\hsize]{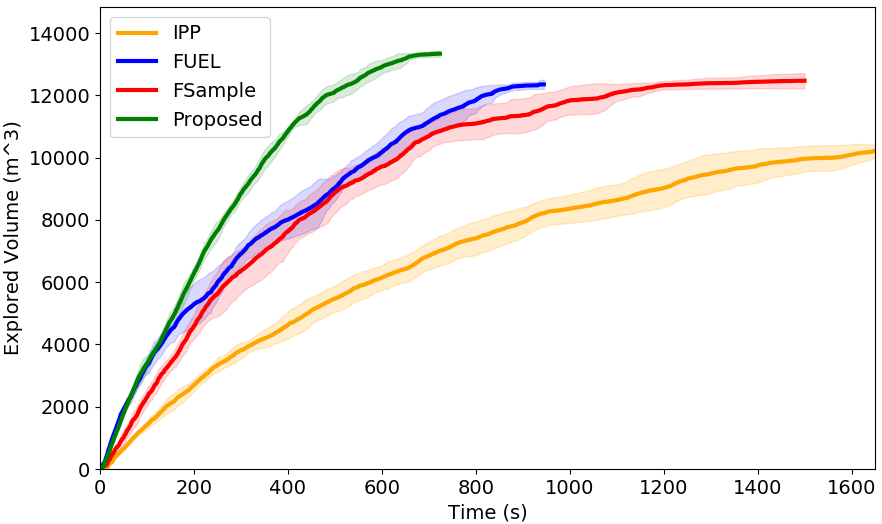}}
		\caption{The exploration results of the \emph{CityBuilding} environment. (a) shows the volumetric mapping result of the environment and trajectories of IPP, FUEL, FSample, and the proposed method from their representative runs. (b) shows the exploration progress of these four methods.}
		\label{CB_statistic}
	\end{minipage}
\end{figure}
The second trial employs the MAV to explore the \emph{CityBuilding} environment. Compared to \emph{LargeMaze}, the \emph{CityBuilding} scenario demonstrates complex geometrical changes along the z-axis, which can validate the basic performance of different exploration methods in real 3-D scenarios. In this trial, we set the map resolution to $0.1$ $m$. The maximum linear velocity, maximum linear acceleration and the angular velocity of the MAV are set to 2.0 $m/s$, 2.0 $m/{s^2}$ and 2.0 $rad/s$, respectively. Fig. \ref{CB_statistic}(a) shows the volumetric mapping results of the environment and the executed trajectories of the MAV. Fig. \ref{CB_statistic}(b) shows the exploration progresses of different methods over the time. Table \ref{table_results_statistic} gives the statistics of four methods in this scenario. It can be found that our method completes the first 90\% coverage of the environment using 609.2$m$ of travel distance and 513.7$s$ of flight time on average. Even the second scenario is more complex and larger than \emph{LargeMaze}, our method can achieve the largest coverage (98.9\% on average) of the environment when it stops to explore. In addition, our method improves the exploration efficiency by more than 58.1\% compared to the other methods. It can be found in Table \ref{table_runtime_comparison}, the computational time used for each planning iteration of FUEL is increased significantly. This is because, in large scenarios (especially open environments), the number of frontiers will increase dramatically, which will thus reduce the exploration efficiency of FUEL. In addition, the computational time of IPP is the longest, resulting in its inability to complete 90\% coverage of the environment within the maximum time limitation. 

\begin{figure}[!htp]
	\begin{minipage}[!htbp]{\linewidth}
		\centering
		\subfigure[Mapping Result and Executed Trajectory of FSMP]{
			\label{PP_a} 
			\includegraphics[width=0.9\hsize]{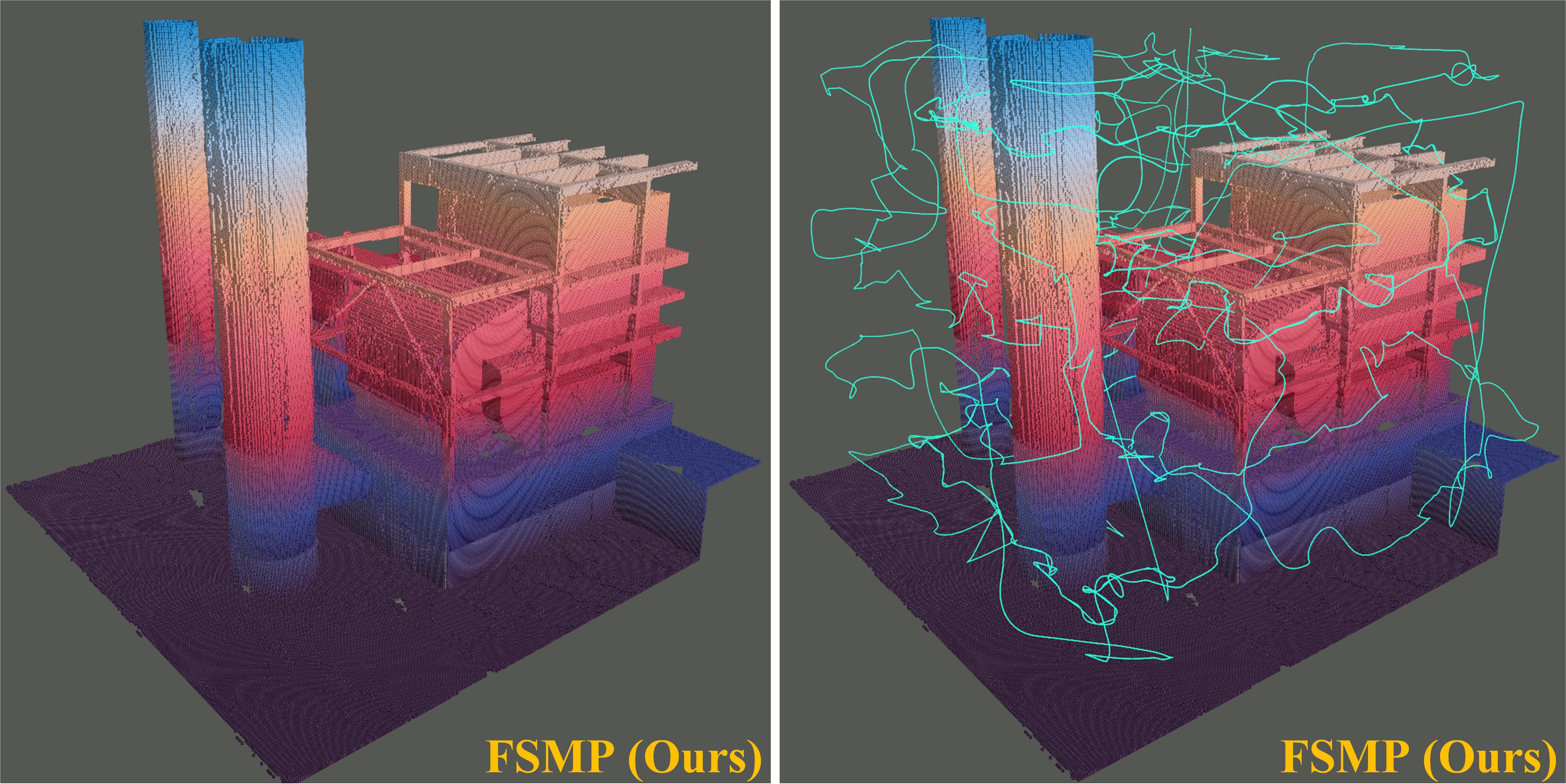}}
		\subfigure[Exploration Progresses]{
			\label{PP_b} 
			\includegraphics[width=0.95\hsize]{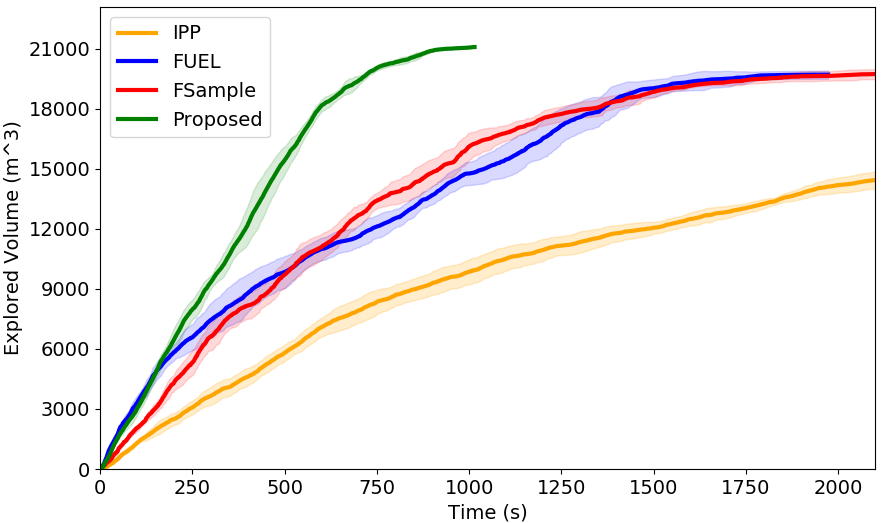}}
		\caption{The exploration results of the \emph{PowerPlant} environment. (a) shows the volumetric mapping result and the executed trajectory of the proposed method from its representative run. (b) shows the exploration progress of four methods.}
		\label{PP_statistic}
	\end{minipage}
\end{figure}
The last trial uses the MAV to explore the \emph{PowerPlant} environment. This scenario is the largest and most complex one of all scenarios. As shown in Fig. \ref{simulation_environments}(c), \emph{PowerPlant} features complex structures and narrow spaces to be explored. In this trial, we set the map resolution to $0.1$ $m$. The maximum linear velocity, maximum linear acceleration and the angular velocity of the MAV are set to 2.0 $m/s$, 2.0 $m/{s^2}$ and 2.0 $rad/s$, respectively. Fig. \ref{PP_statistic}(a) shows the volumetric mapping result and the executed trajectory of FSMP. Our method completes the first 90\% coverage of the environment using 609.2$m$ of travel distance and 791.2$s$ of flight time on average. Fig. \ref{PP_statistic}(b) shows the explored volume over the time of four methods. It can be found in Table \ref{table_results_statistic} that our method improves the exploration efficiency by more than 124.4\% compared to the other methods. In addition, IPP, FUEL, and FSample can not reach 95\% coverage of the environment within the maximum time limitation. However, our method can achieve 98.1\% coverage on average of the environment, demonstrating our method possesses good scalability. 

\begin{table}[!t]\centering
	\caption{Computation time used for each planning iteration} \label{table_runtime_comparison}
	\begin{tabular}{|c|c|cccc|}
		\hline
		\multirow{2}{*}{\textbf{Scenario}} & \multirow{2}{*}{\textbf{Method}} & \multicolumn{4}{c|}{\textbf{Time} (ms)}  \\ \cline{3-6} &                                  
		& \multicolumn{1}{c|}{\textbf{Avg}}    & \multicolumn{1}{c|}{\textbf{Std}}   & \multicolumn{1}{c|}{\textbf{Max}}    & \textbf{Min}     \\ \hline
		\multirow{4}{*}{\emph{LargeMaze}}      
		& IPP                               & \multicolumn{1}{c|}{1171.5}          & \multicolumn{1}{c|}{370.3}           & \multicolumn{1}{c|}{3106.9}          & 184.1         \\ \cline{2-6}
		
		& FUEL                              & \multicolumn{1}{c|}{481.0}               & \multicolumn{1}{c|}{498.0}          & \multicolumn{1}{c|}{2491.0}            & 4.1         \\ \cline{2-6}
		
		& FSample                           & \multicolumn{1}{c|}{817.6}          & \multicolumn{1}{c|}{2287.0}             & \multicolumn{1}{c|}{16933.1}         & 71.3              \\ \cline{2-6}
		
		& Proposed                          & \multicolumn{1}{c|}{\textbf{22.2}}    & \multicolumn{1}{c|}{\textbf{12.7}}    & \multicolumn{1}{c|}{\textbf{94.1}}  & \textbf{0.22}  \\ \hline
		
		\multirow{4}{*}{\emph{CityBuilding}}         
		& IPP                               & \multicolumn{1}{c|}{1904.3}                & \multicolumn{1}{c|}{507.9}           & \multicolumn{1}{c|}{4162.9}          & 485.2           \\ \cline{2-6}
		
		& FUEL                              & \multicolumn{1}{c|}{625.7}             & \multicolumn{1}{c|}{492.0}             & \multicolumn{1}{c|}{3003.1}          & 31.6           \\ \cline{2-6}
		
		& FSample                           & \multicolumn{1}{c|}{675.5}            & \multicolumn{1}{c|}{1544.4}            & \multicolumn{1}{c|}{13336.0}            & 97.5              \\ \cline{2-6}
		
		& Proposed                          & \multicolumn{1}{c|}{\textbf{107.7}}     & \multicolumn{1}{c|}{\textbf{97.1}}   & \multicolumn{1}{c|}{\textbf{380.8}}    & \textbf{0.45}  \\ \hline
		
		\multirow{4}{*}{\emph{PowerPlant}}           
		& IPP                               & \multicolumn{1}{c|}{1677.79}          & \multicolumn{1}{c|}{1727.84}          & \multicolumn{1}{c|}{5618.2}          & 173.8             \\ \cline{2-6}
		
		& FUEL                              & \multicolumn{1}{c|}{1469.6}          & \multicolumn{1}{c|}{1727.8}          & \multicolumn{1}{c|}{15265.1}          & 18.4              \\ \cline{2-6}
		
		& FSample                           & \multicolumn{1}{c|}{629.6}          & \multicolumn{1}{c|}{2333.4}          & \multicolumn{1}{c|}{35037.0}          & 4.3    \\ \cline{2-6}
		
		& Proposed                          & \multicolumn{1}{c|}{\textbf{164.5}}      & \multicolumn{1}{c|}{\textbf{126.9}}    & \multicolumn{1}{c|}{\textbf{562.9}}     & \textbf{0.21}         \\ \hline
	\end{tabular}
\end{table}

\subsection{Real-world Experiments}\label{sub_section:Real-world-Experiments}
To further validate our planner, we implement FSMP and FUEL, respectively, on a self-built quadrotor MAV (see Fig. \ref{robot_platform}(b)) to explore an outdoor environment with cluttered obstacles. As shown in Fig. \ref{real-world-scenes}(a), the environment to be explored is with a size of 20 $\times$ 26 $\times$ 3 $m$$^3$. For both FSMP and FUEL, the maximum linear velocity, maximum linear acceleration and the maximum angular velocity of the MAV are set to 2.0 $m/s$, 2.0 $m/{s^2}$ and 1.5 $rad/s$, respectively. The resolution of the volumetric map is set to 0.1 $m$. The LiDAR-based odometry, Point-LIO \cite{he2023point}, is employed in the experiments for the MAV pose estimation. Note that there are no external devices used for the real-world experiments and all algorithms are running on-board the MAV in real time. To achieve fair evaluation, we set the maximum time used for exploration to 120 $s$.
	
The volumetric mapping results and the executed trajectories of FSMP and FUEL are shown in Fig. \ref{real-world-scenes}(b) and (c), respectively. The explored volume over time of FSMP and FUEL are shown in Fig. \ref{real-world-scenes}(d). It can be seen that our method can complete the first 90\% coverage of the environment in 85.5 $s$, while FUEL takes 99.5 $s$, which is 16.3\% longer than our method. In addition, FSMP can achieve 96.0\% coverage (i.e., 1498.3 $m^3$ out of 1560 $m^3$) of the environment within the maximum time limit, whereas FUEL can only achieve 93.6\% coverage (i.e., 1459.9 $m^3$ out of 1560 $m^3$). 

Overall, the simulation and real-world results demonstrate the efficient and effective performance of our method in various environments.

\begin{figure}[!t]
	\centering
	\includegraphics[width=0.8\hsize]{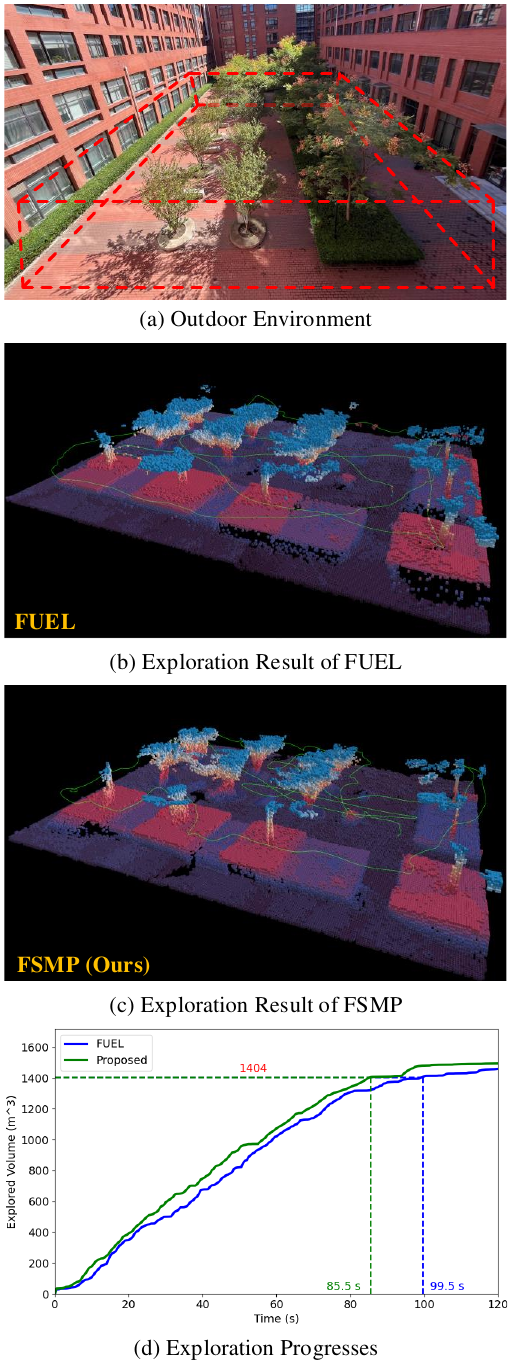}\vspace{-0.1cm}
	\caption{Real-world experiment in an outdoor environment with cluttered obstacles.}
	\label{real-world-scenes}
\end{figure}
\section{Conclusion}\label{section:Conclusion}
In this paper, a novel frontier-sampling-mixed exploration planner (FSMP) is developed by integrating advantages of the frontier-based and sampling-based strategies. According to the simulation studies, FSMP can achieve superior exploration performance over some state-of-the-art autonomous exploration methods in terms of exploration efficiency (more than 9.3\% faster in the benchmark environment and 124.4\% faster in the most complex and largest 3-D environment), computational time (several times faster at least), and explored volume (only our method achieved more than 98\% coverage of the last 3-D environment). Finally, the real-world experiments demonstrated that our planner is effective for exploring large and 3-D environments and can be executed in real-time on-board a self-developed MAV.

The main limitations of FSMP are twofold. First, we assume that the MAV can perfectly localize itself, as most methods do. However, state estimation errors are ubiquitous and can not be ignored, especially in real-world environments. Large state estimation errors can significantly degrade the exploration performance or even lead to task failure. Second, we assume that the environment to be explored is static. However, in practical scenarios, dynamic obstacles are frequently encountered, leading to insufficient coverage. In our future work, we plan to consider odometry drifts in FSMP to improve its robustness under localization uncertainty and devise specific strategies for addressing exploration problems in dynamic environments. Additionally, we are working on extending FSMP for fast exploration of large field environments with multiple MAVs.

\bibliographystyle{IEEEtran}

\bibliography{Bibliography/bibFile}

\end{document}